\documentclass[twoside]{article}

\usepackage{amsmath}
\usepackage{amsthm}
\usepackage{amssymb}
\usepackage{mathtools}
\usepackage{mathrsfs}
\usepackage{upgreek}
\usepackage{mleftright}

\usepackage[pdftex]{hyperref}
\usepackage[sf,bf]{titlesec}
\usepackage{bm}
\usepackage{bbold}
\usepackage{euscript}
\usepackage{url}            
\usepackage{booktabs}       
\usepackage{amsfonts}       
\usepackage{nicefrac}       
\usepackage{microtype}      
\usepackage{mathrsfs}
\usepackage{color}
\usepackage{tabu}
\usepackage{multirow}
\usepackage{float}
\usepackage{wrapfig}
\usepackage{booktabs}
\usepackage[ruled,vlined]{algorithm2e}
\DontPrintSemicolon
\usepackage{color}
\usepackage{tabu}
\usepackage{comment}
\usepackage{graphicx}
\usepackage{epstopdf}
\usepackage{subcaption}
\usepackage{framed}
\usepackage{enumitem}
\usepackage{textcomp}
\usepackage{setspace}
\usepackage[noabbrev]{cleveref}
\usepackage{latexsym}
\usepackage{tcolorbox}
\usepackage[round,sort&compress]{natbib}
\usepackage{alltt}
\usepackage{stmaryrd}
\usepackage{xspace}
\usepackage{euscript}
\usepackage{xfrac}    
\usepackage{nicefrac} 
\usepackage{tikz}
\usepackage{tikz-qtree}
\usepackage{todonotes}
\usepackage{diagbox}

\usepackage[margin=1in]{geometry}
\usepackage{fancyhdr}
\fancypagestyle{plain}{\fancyhf{} 
	
	\fancyfoot[C]{\textsf{\thepage}}
}

\definecolor{mydarkblue}{rgb}{0,0.08,0.45}
\hypersetup{ %
	colorlinks=true,
	linkcolor=mydarkblue,
	citecolor=mydarkblue,
	filecolor=mydarkblue,
	urlcolor=mydarkblue}

\newcommand{\calL}{\mathcal{L}}

\newcommand{\scrO}{\mathscr{O}}
\newcommand{\scrP}{\mathscr{P}}

\newcommand{\scrW}{\mathscr{W}}

\newcommand{\euB}{\EuScript{B}}

\newcommand{\euD}{\EuScript{D}}
\newcommand{\euE}{\EuScript{E}}
\newcommand{\euF}{\EuScript{F}}
\newcommand{\euG}{\EuScript{G}}

\newcommand{\euK}{\EuScript{K}}
\newcommand{\euL}{\EuScript{L}}

\newcommand{\euN}{\EuScript{N}}

\newcommand{\euQ}{\EuScript{Q}}
\newcommand{\euR}{\EuScript{R}}

\newcommand{\euU}{\EuScript{U}}
\newcommand{\euV}{\EuScript{V}}
\newcommand{\euW}{\EuScript{W}}
\newcommand{\euX}{\EuScript{X}}
\newcommand{\euY}{\EuScript{Y}}

\newcommand{\Var}{\mathrm{Var}}
\newcommand{\Cov}{\mathrm{Cov}}
\newcommand{\Ex}{\mathbb{E}}
\newcommand{\Prob}{\mathbb{P}}

\newcommand{\One}{\bm{1}}

\newcommand{\QQ}{\mathbb{Q}}
\newcommand{\RR}{\mathbb{R}}

\newcommand{\Rp}{\RR_+}
\newcommand{\Rpp}{\RR_{++}}
\newcommand{\bbS}{\mathbb{S}}

\newcommand{\NN}{\mathbb{N}}

\newcommand{\diff}{\mathrm{d}}

\newcommand{\e}{\mathrm{e}}

\DeclareMathOperator*{\argmin}{argmin}

\DeclareMathOperator*{\minimize}{minimize}

\DeclareMathOperator*{\argmax}{argmax}

\DeclareMathOperator*{\Diag}{Diag}

\newcommand{\bp}{\bm{p}}

\newcommand{\sumK}{\sum_{k=1}^K}

\newcommand{\sumn}{\sum_{i=1}^n}

\newcommand{\blambda}{\bm{\lambda}}

\newcommand{\tSigma}{\widetilde{\Sigma}}

\newcommand{\tr}{\mathrm{tr}}

\newcommand{\sfB}{\mathsf{B}}

\newcommand{\sfD}{\mathsf{D}}

\newcommand{\sfG}{\mathsf{G}}

\newcommand{\sfQ}{\mathsf{Q}}

\newcommand{\sfS}{\mathsf{S}}

\newcommand{\sfW}{\mathsf{W}}

\newcommand{\oRR}{\overline{\RR}}

\newcommand{\Sigmahat}{\widehat{\Sigma}}
\newcommand{\Sigmabar}{\overline{\Sigma}}

\newcommand{\muhat}{\widehat{\mu}}
\newcommand{\mubar}{\overline{\mu}}

\newcommand{\midd}{\,|\kern-0.25ex|\,}
\newcommand{\setn}{\llbracket n\rrbracket}

\newcommand{\setK}{\llbracket K\rrbracket}

\newcommand{\set}[1]{\llbracket #1\rrbracket}
\newcommand{\KL}{\mathrm{KL}}

\newcommand{\dotp}[2]{\langle #1, #2\rangle}
\newcommand{\dotpF}[2]{\left\langle #1, #2\right\rangle_{\rm F}}

\newcommand{\OT}{\mathsf{OT}}

\newcommand{\Probhat}{\widehat{\mathbb{P}}}
\newcommand{\ac}{\mathrm{ac}}
\newcommand{\teuW}{\widetilde{\euW \mkern 1mu}}
\newcommand{\teuU}{\widetilde{\euU \mkern 1mu}}

\newcommand{\sfb}{\mathsf{b}}
\newcommand{\DD}{\mathbb{D}}

\newcommand{\half}{\sfrac12}

\DeclareMathAlphabet\rsfscr{U}{rsfso}{m}{n}

\theoremstyle{plain}
\newtheorem{theorem}{Theorem}[section]
\newtheorem{proposition}[theorem]{Proposition}
\newtheorem{lemma}[theorem]{Lemma}
\newtheorem{corollary}[theorem]{Corollary}
\theoremstyle{definition}
\newtheorem{definition}[theorem]{Definition}

\theoremstyle{remark}
\newtheorem{remark}[theorem]{Remark}

\crefname{assumption}{Assumption}{Assumptions}
\Crefname{assumption}{Assumption}{Assumptions}
\crefname{problem}{Problem}{Problems}
\Crefname{problem}{Problem}{Problems}
\crefname{example}{Example}{Examples}
\Crefname{example}{Example}{Examples}

\let\le\leqslant
\let\ge\geqslant

\let\hat\widehat
\let\tilde\widetilde
\let\bar\overline

\makeatletter
\DeclareFontFamily{OMX}{MnSymbolE}{}
\DeclareSymbolFont{MnLargeSymbols}{OMX}{MnSymbolE}{m}{n}
\SetSymbolFont{MnLargeSymbols}{bold}{OMX}{MnSymbolE}{b}{n}
\DeclareFontShape{OMX}{MnSymbolE}{m}{n}{
	<-6>  MnSymbolE5
	<6-7>  MnSymbolE6
	<7-8>  MnSymbolE7
	<8-9>  MnSymbolE8
	<9-10> MnSymbolE9
	<10-12> MnSymbolE10
	<12->   MnSymbolE12
}{}
\DeclareFontShape{OMX}{MnSymbolE}{b}{n}{
	<-6>  MnSymbolE-Bold5
	<6-7>  MnSymbolE-Bold6
	<7-8>  MnSymbolE-Bold7
	<8-9>  MnSymbolE-Bold8
	<9-10> MnSymbolE-Bold9
	<10-12> MnSymbolE-Bold10
	<12->   MnSymbolE-Bold12
}{}

\let\llangle\@undefined
\let\rrangle\@undefined
\DeclareMathDelimiter{\llangle}{\mathopen}%
{MnLargeSymbols}{'164}{MnLargeSymbols}{'164}
\DeclareMathDelimiter{\rrangle}{\mathclose}%
{MnLargeSymbols}{'171}{MnLargeSymbols}{'171}
\makeatother

\newcommand{\lrangle}[1]{\left\llangle #1 \right\rrangle}
\renewcommand{\dotpF}[2]{\lrangle{#1, #2}_{\rm F}} 
\newcommand{\norm}[1]{\left\lVert#1\right\rVert}
\newcommand{\euclidnorm}[1]{\left\lVert#1\right\rVert_2}
\newcommand{\vecnorm}[2]{\left\| #1 \right\|_{{#2}}}
\newcommand{\matsnorm}[2]{|\kern-0.25ex|\kern-0.25ex| #1 |\kern-0.25ex|\kern-0.25ex|_{{#2}}}

\newcommand{\fronorm}[1]{\matsnorm{#1}{\mathrm{F}}}

\newcommand{\onenorm}[1]{\vecnorm{#1}{1}}

\newcommand{\Law}{\mathrm{Law}}
\renewcommand{\left}{\mleft}
\renewcommand{\right}{\mright}

\begin{document}
	\title{\sffamily Wasserstein Distributionally Robust Optimization \\with Wasserstein Barycenters}
	\author{
		Tim Tsz-Kit Lau\thanks{
			Department of Statistics and Data Science, Northwestern University, Evanston, IL 60208, USA; Email: \href{mailto:timlautk@u.northwestern.edu}{\texttt{timlautk@u.northwestern.edu}}.}
		    \and
		    Han Liu\thanks{Department of Computer Science and Department of Statistics and Data Science, Northwestern University, Evanston, IL 60208, USA; Email: \href{mailto:hanliu@northwestern.edu}{\texttt{hanliu@northwestern.edu}}.}
	}
	
	\maketitle

	\numberwithin{equation}{section}
	
	\begin{abstract}		
		In many applications in statistics and machine learning, the availability of data samples from multiple possibly heterogeneous sources has become increasingly prevalent. On the other hand, in distributionally robust optimization, we seek data-driven decisions which perform well under the most adverse distribution from a nominal distribution constructed from data samples within a certain discrepancy of probability distributions. However, it remains unclear how to achieve such distributional robustness in model learning and estimation when data samples from multiple sources are available. In this work, we propose constructing the nominal distribution in optimal transport-based distributionally robust optimization problems through the notion of Wasserstein barycenter as an aggregation of data samples from multiple sources. Under specific choices of the loss function, the proposed formulation admits a tractable reformulation as a finite convex program, with powerful finite-sample and asymptotic guarantees. As an illustrative example, we demonstrate with the problem of distributionally robust sparse inverse covariance matrix estimation for zero-mean Gaussian random vectors that our proposed scheme outperforms other widely used estimators in both the low- and high-dimensional regimes. 
	\end{abstract}

	\section{Introduction}	
	In various statistical and machine learning applications, data samples are collected from multiple sources, which can be viewed as samples drawn from multiple data distributions. A notable example is federated learning \citep{kairouz2021advances,mcmahan2017communication}, in which many users collaboratively learn a common model but the samples collected by the clients might have highly heterogeneous distributions. This distribution heterogeneity leads to difficulty in building a robust model in two aspects: (i) how to aggregate estimations of the distributions with data samples from these sources; (ii) how to perform robust estimation with this data aggregation given distributional uncertainty. 
	
	In practice, the first issue is usually dealt with by simply taking a simple (weighted) average of the distribution estimates, whereas the second one is tackled by minimizing the weighted aggregate loss with possibly different weights. However, the mixture distribution constructed from the weighted average of distributions does not take into account the geometric structure of data samples, thus failing to well summarize the characteristics from all sources. In this work, we consider the notion of barycenter (a.k.a.~Fr\'{e}chet mean) in the space of probability distributions endowed with the Wasserstein distance, called \emph{Wasserstein barycenter} \citep{agueh2011barycenters}, which is a nonlinear interpolation between distributions.  
	
	To perform robust estimation of models against distributional uncertainty, \emph{distributionally robust optimization} \citep[DRO;][]{delage2010distributionally,goh2010distributionally,wiesemann2014distributionally} has been shown to be a powerful modeling framework, which has aroused much attention in the machine learning community lately attributed to its connections to generalization, regularization and robustness.

	\paragraph{Contributions.}
	We thus propose a unified approach to overcome these two aspects of difficulty. We first construct an aggregate distribution of multiple data distributions through Wasserstein barycenter, followed by using it to define an ambiguity set in a DRO problem, which is a family of distributions lying within a certain Wasserstein distance from this Wasserstein barycenter, coined the \emph{Wasserstein barycentric ambiguity set}. We hence introduce \emph{Wasserstein Barycentric DRO} (WBDRO) as a general aggregate data-driven decision making framework with distributional robustness against uncertainty arising in multiple possibly heterogeneous unknown true distributions.  
	
	We establish finite-sample guarantees and asymptotic consistency results for WBDRO. We also consider an approximation of the Wasserstein ambiguity set by characterizing it using only the first two moments of the family of distributions and those of the nominal distribution, called the \emph{Gelbrich ambiguity set}. We further extend this construction in the case of multiple nominal distributions using the $2$-Wasserstein barycenter. We also exemplify WBDRO through distributionally robust maximum likelihood estimation for sparse inverse covariance matrices of zero-mean Gaussian random vectors, which numerically outperforms other widely-used estimators.

	\subsection{Related Work}	
	\paragraph{Distributionally Robust Optimization.}
	As a powerful modeling framework, DRO has recently found a wide range of applications in statistics and machine learning \citep{shafieezadeh2015distributionally,shafieezadeh2019regularization,nguyen2022distributionally,duchi2021learning,duchi2021statistics,blanchet2019robust,bertsimas2022bootstrap,li21distributionally,nguyen2021bridging,taskesen2021sequential}, signal processing \citep{shafieezadeh2018wasserstein}, portfolio selection and maximization \citep{blanchet2021distributionally,nguyen2021robustifying,nguyen2021mean,obloj2021distributionally}, etc. 	
	One key component of DRO is the choice of data-driven ambiguity sets, which can be defined through $f$-divergence \citep{duchi2021learning,duchi2021statistics,ben2013robust}, Wasserstein distance \citep{pflug2007ambiguity,gao2016distributionally,gao2017wasserstein,gao2020finite,blanchet2019data}, generalized moment constraints \citep{delage2010distributionally,goh2010distributionally,wiesemann2014distributionally,bertsimas2018data}, maximum mean discrepancy \citep[MMD;][]{staib2019distributionally}, etc. Tractable reformulation as finite convex programs are available for DRO problems with these different ambiguity sets. 	
	We refer to \citet{levy2020large,li21distributionally,jin2021non,carmon2022distributionally,yu2022fast,haddadpour2022learning} for recent advances in the computational perspectives of DRO, and \citet{zhen2021mathematical} for a review of the mathematical foundations of DRO.

	\paragraph{Notion of Mean Distributions.}
	Fr\'{e}chet mean or barycenter in different metric spaces, as a notion of mean distributions, has long been a central object in statistical analysis. Notably, a (weighted) average of distributions on $\RR^d$ is a barycenter in Euclidean space. The 	
	Wasserstein barycenter \citep{agueh2011barycenters,kroshnin2018frechet} is a more appropriate notion of mean distributions since the geometric structure of the distributions can be considered \citep[see e.g.,][]{backhoff2018bayesian}. 	
	This notion has already appeared in various applications in statistics and machine learning \citep{bishop2014information,bishop2021network,yang2021clustering,schmitz2018wasserstein,bigot2019data,backhoff2018bayesian,srivastava2018scalable}. In particular, the work \citet{alvarez2018wide} shares similar motivation to ours, which is to perform consensus-based estimation combining several estimations of probability distributions.

	\paragraph{Learning with Data from Multiple Sources.}
	Modern machine learning applications involve the use of data collected from multiple sources. Such examples include federated learning \citep{mcmahan2017communication,wang2021field,kairouz2021advances}, (multiple-source) domain adaptation \citep{mansour2021theory,zhang2021multiple}, information fusion and network consensus \citep{bishop2014information,bishop2021network}. However, the consensus problem indeed has a much longer history \citep[see e.g.,][]{degroot1974reaching}. 
	
	A more detailed discussion on other related prior work can be found in \Cref{sec:add_related}.

	\section{Preliminaries}
	\label{sec:prelim}
	
	\paragraph{Notation.} 
	We denote by $I_d\in\RR^{d\times d}$ the $d\times d$ identity matrix and $\One_d\in\RR^d$ the $d$-dimensional all-one vector. 
	The subscripts for dimensions are suppressed if they are clear from context. We define $\setn \coloneqq \{1, \ldots, n\}$ for $n\in\NN^*$. Let $\bbS^d_{++}$ (resp.~$\bbS^d_+$) denote the set of symmetric positive (resp.~semi-)definite matrices. 	
	$\scrP(\euX)$ is the set of Borel probability measures over the Polish space $\euX$, $\scrP_k(\euX)$ is the set of probability measures over $\euX$ with finite $k$-order moments, and $\scrP_k^{\ac}(\euX)$ is the set of absolutely continuous probability measures over $\euX$ (w.r.t.~the Lebesgue measure) with finite $k$-order moments. 	
	The set $\triangle^{d} \coloneqq \{\bp\in[0,+\infty)^d:\dotp{\bp}{\One_d} = 1\}$ is the $(d-1)$-dimensional probability simplex, where $\dotp{\cdot}{\cdot}$ is the usual inner product. 
	We denote by $\euN(\mu, \Sigma)$ a Gaussian distribution with mean $\mu\in\RR^d$ and covariance matrix $\Sigma\in\bbS_{+}^d$, and 	$\updelta_x$ a Dirac measure at point $x\in\euX$.

	\paragraph{Optimal Transport.}
	We introduce several notions from optimal transport (OT) used throughout the whole paper, which can be found in various monographs on the subject  \citep{villani2003topics,villani2009optimal,santambrogio2015optimal,ambrosio2021lecture,figalli2021invitation,peyre2019computational}. We also refer to \citet{panaretos2020invitation,panaretos2019statistical,peyre2019computational} for comprehensive reviews of recent advances of optimal transport in machine learning and statistics. Let $\Omega\subseteq\RR^m$ be a closed convex set. 
	For $p\in\left[1,+\infty\right)$, the $p$-Wasserstein distance between two probability measures $\rho, \nu\in \scrP_p(\Omega)$ is defined by 
	\begin{equation}\label{eqn:Wasserstein}
		\sfW_p(\rho, \nu) \coloneqq \left(\inf_{\pi\in\Pi(\rho, \nu)}\int_{\Omega\times\Omega} \norm{x - y}^p\,\diff\pi(x, y)\right)^{\negthickspace\sfrac1p},
	\end{equation}
	where $\norm{\cdot}$ is the Euclidean norm on $\RR^m$, and $\Pi(\rho, \nu)$ denotes the set of joint distributions on $\RR^m\times\RR^m$ with $\rho$ and $\nu$ as marginals. 
	The $p$-Wasserstein distance is a distance on the space $\scrP_p(\Omega)$ \citep[see e.g.,][Theorem 3.1.5]{figalli2021invitation}. 	
	We call the metric space $\scrW_p(\Omega) \coloneqq (\scrP_p(\Omega), \sfW_p)$ the \emph{$p$-Wasserstein space} \citep{ambrosio2005gradient}. 
	
	The notion of Wasserstein barycenter \citep{agueh2011barycenters} can be viewed as the mean of probability distributions in the Wasserstein space. 	For $p\in\left[1,+\infty\right)$, the \emph{$p$-Wasserstein barycenter} of $\Prob\in\scrW_p(\scrP_p(\Omega))$ is defined by 		
	\begin{equation}\label{eqn:pop_barycenter}
		\sfb_p(\Prob) \coloneqq \argmin_{\nu\in\scrP_p(\RR^m)} \Ex_{\rho\sim\Prob} \left[\sfW_p^p(\nu, \rho)\right], 
	\end{equation}
	where $\rho\in\scrP_p(\Omega)$ is a random measure with distribution $\Prob$. 
	If we take $\Prob = \sumK \lambda_k\updelta_{\rho_k}$ in \eqref{eqn:pop_barycenter}, where $\blambda = (\lambda_k)_{k\in\setK}\in\triangle^{K}$, we recover the \emph{$\blambda$-weighted empirical $p$-Wasserstein barycenter}, defined by \[\hat{\sfb}_{\blambda, p}(\rho_1,\ldots,\rho_K) \coloneqq \argmin_{\nu\in\scrP_p(\RR^m)}\, \sumK \lambda_k\sfW_p^p(\nu, \rho_k). \]
	Thus, we use the term $p$-Wasserstein barycenter to refer to both empirical and population $p$-Wasserstein barycenters whenever it is clear from context. Note that Wasserstein barycenters do not always exist, and might not be unique if exist. Technical conditions for their existence and uniqueness are studied in e.g., \citet{agueh2011barycenters,le2017existence}.

	\paragraph{Distributionally Robust Optimization.}
	In DRO, we investigate a learning problem under distributional uncertainty which is casted as a generic expected loss minimization framework. The loss function $\ell\colon\RR^m\to\oRR\coloneqq\RR\cup\{\pm\infty\}$ is a function of the the uncertainty vector $\xi\in\RR^m$ whose distribution $\Prob$ is supported on $\Xi\subseteq\RR^m$. The \emph{risk} (or \emph{expected loss}) of a decision $\ell\in\euL$ is defined as 
	\[\euR_\Prob(\ell) \coloneqq \Ex_{\xi\sim\Prob}[\ell(\xi)], \]
	where $\euL$ is the set of all admissible loss functions. The \emph{optimal risk} is then defined as the infimum of the risk over $\euL$. However, $\Prob$ is often unknown in practice except for some limited statistical and structural information about it. We thus assume that $\Prob$ is known to lie in an \emph{ambiguity set} $\euU_\varepsilon(\Probhat)$, which is a ball of radius $\varepsilon\ge0$ in $\scrP(\Xi)$ centered at the nominal distribution $\Probhat$ in some discrepancy between probability distributions. We can then define the \emph{worst-case risk} of $\ell\in\euL$ by 
	\begin{equation}\label{eqn:worst-case_risk}
		\euR_{\euU_\varepsilon(\Probhat)}(\ell) \coloneqq \sup_{\Prob\in\euU_\varepsilon(\Probhat)} \euR_\Prob(\ell). 
	\end{equation}
	Such a nominal distribution $\Probhat$ is usually constructed from a set of observed data $\euD \coloneqq \{z_i\}_{i=1}^n\subset\RR^m$, e.g., its empirical measure $\frac1n\sumn\updelta_{z_i}$. 
	The distributionally robust optimization (DRO) problem seeks decisions achieving the \emph{optimal worst-case risk} 
	\begin{equation}\label{eqn:dro}
		\euR_{\euU_\varepsilon(\Probhat)}(\euL) \coloneqq \inf_{\ell\in\euL} \euR_{\euU_\varepsilon(\Probhat)}(\ell). 
	\end{equation}
	
	\begin{remark}
		Note that if the loss function is parameterized by the decision $x\in\euX\subseteq\RR^d$, i.e., $\ell\colon\euX\times\RR^m\to\oRR$, then the risk can be defined in terms of $x$ as $\euR_\Prob(x) \coloneqq \Ex_{\xi\sim\Prob}[\ell(x, \xi)]$. The worst-case risk and the worst-case optimal risk can be defined similarly. 
	\end{remark}

	In this paper, we consider the ambiguity set defined via the $p$-Wasserstein distance. Then, the \emph{$p$-Wasserstein ambiguity set} is defined by 
	\[\euW_{\varepsilon, p}(\Probhat) \coloneqq \{\QQ\in\scrP_p(\Xi) : \sfW_p(\QQ, \Probhat) \le \varepsilon\}, \]
	where $\Xi\subseteq\RR^m$ is a closed set which is known to contain the support of the unknown true distribution $\Prob^\star$ and $\varepsilon\ge0$. 
	Such a DRO formulation is called the \emph{Wasserstein DRO} (WDRO).

	\section{Learning with Aggregation of Multiple Distributions}
	\label{sec:learning_multiple_sources}	
	\paragraph{Problem Formulation.}
	In different centralized model learning scenarios with data from multiple sources, such as federated learning, the learning objective can usually be casted as a \emph{stochastic composition optimization} problem \citep[see e.g.,][]{yuan2022what,wang2021field}: 
	\begin{equation}\label{eqn:fo_part}
		\minimize_{x\in\euX} \ F(x) \coloneqq \Ex_{k\sim\DD} \left[f_k(x) \right], \quad \text{where} \quad f_k(x) \coloneqq \Ex_{\xi\sim\Prob_k}[\ell(x, \xi)], 
	\end{equation}
	where $x\in\euX$ is the parameter of the global model for some closed convex set $\euX\subseteq\RR^d$, $\xi\in\Xi$ is a random vector representing an input-output pair for some sample space $\subseteq\RR^m$, $f_k\colon\RR^d\to\oRR$ is the local objective function of the $k$th source, $\Prob_k$ is the distribution associated to the $k$th source, and $\DD$ is a distribution supported on the set of sources $\euK$. 
	Assuming there is only a finite number of $K$ sources, i.e., $\euK = \setK$, then the objective in \eqref{eqn:fo_part} can be written as $F_{\blambda}(x) \coloneqq \sumK \lambda_k f_k(x)$, where $\DD$ is taken to be a categorical distribution with probabilities $\blambda= (\lambda_k)_{k\in\setK}\in\triangle^K$.

	Usually, each data source $k$ has a finite number of local samples, denoted by $\euD_k = (z_{k,1}, \ldots, z_{k,n_k})$, where $n_k$ is the sample size of the $k$th source and $N\coloneqq\sumK n_k$. Using their empirical distributions $\Probhat_k \coloneqq \frac1{n_k}\sum_{i=1}^{n_k} \updelta_{z_{k,i}}$, with $\hat{f}_k(x) \coloneqq \Ex_{\xi\sim\Probhat_k}[\ell(x,\xi)]$, we usually solve the following \emph{empirical risk minimization} (ERM) problem in practice: 
	\begin{equation}\label{eqn:emp_fo_weighted}
		\minimize_{x\in\euX} \ \hat{F}_{\blambda}(x) \coloneqq \sumK \lambda_k \hat{f}_k(x) = \sumK \frac{\lambda_k}{n_k} \sum_{i=1}^{n_k} \ell(x, z_{k,i}), 
	\end{equation}
	Note that $\blambda$ is usually taken as the uniform distribution over the numbers of samples from the sources, i.e., $\lambda_k = n_k/N$, so that the ERM objective \eqref{eqn:emp_fo_weighted} is amount to an ERM objective with the union of all the local data samples.

	However, as argued by \citet{mohri2019agnostic,ro2021communication}, this choice of the uniform distribution is questionable since there is often a mismatch between the target distribution (for which the centralized model is learned) and the mixture distribution $\sumK n_k \Prob_k/N$. Instead, the target distribution is better expressed as a $\blambda$-mixture of $\Prob_1, \ldots, \Prob_K$, i.e., $\Prob_{\blambda}\coloneqq \sumK \lambda_k\Prob_k$ for some $\blambda\in\triangle^K$. 
	Then, with the $\blambda$-mixture of the empirical distributions $\Probhat_{\blambda} \coloneqq \sumK \lambda_k\Probhat_k = \sumK \frac{\lambda_k}{n_k}\sum_{i=1}^{n_k}\updelta_{z_{k,i}}$, it is not hard to see that the objective in \eqref{eqn:emp_fo_weighted} is equivalent to $\Ex_{\xi\sim\Probhat_{\blambda}}[\ell(x,\xi)]$.

	\paragraph{Stochastic Barycentric Optimization.}
	Let us recall that $\Prob_{\blambda}$ is the $\blambda$-weighted Euclidean barycenter of the distributions $\Prob_1, \ldots, \Prob_K$. Leveraging this important fact, we consider more generally a $\blambda$-weighted barycenter $\hat\sfb_{\blambda}(\Prob_1, \ldots, \Prob_K)$ of $\Prob_1, \ldots, \Prob_K$ defined via some discrepancy between distributions such as the Wasserstein distance, so that \eqref{eqn:fo_part} can be formulated with the objective 
	\[F_{\blambda}^\sfb(x) \coloneqq \Ex_{\xi\sim\hat\sfb_{\blambda}(\Prob_1, \ldots, \Prob_K)}[\ell(x, \xi)], \]
	which we refer to as a \emph{stochastic barycentric optimization (SBO)} problem.    
	With the data samples $\euD_1, \ldots, \euD_K$, we also have its surrogate objective defined with the empirical distributions $\Probhat_1, \ldots, \Probhat_K$, given by 
	\[\hat{F}_{\blambda}^\sfb(x) \coloneqq \Ex_{\xi\sim\hat\sfb_{\blambda}(\Probhat_1, \ldots, \Probhat_K)}[\ell(x, \xi)]. \]

	Unfortunately, except for the case of the Euclidean barycenter, $\hat{F}_{\blambda}^\sfb$ usually cannot be expressed as a finite sum. This appears to be unfavorable computationally compared to \eqref{eqn:emp_fo_weighted}. To solve it computationally, one can resort to  ERM by drawing samples from the barycenter of empirical distributions. 
	However, solving such an ERM problem requires (i) the computation of a barycenter; and (ii) sampling from such a barycenter. 
	In the case of the Wasserstein barycenter, these tasks could be computationally intensive \citep{altschuler2022wasserstein} or not well addressed until recently \citep{daaloul2021sampling}. 
	However, the choice of the Wasserstein barycenter over the Euclidean barycenter in SBO is justified in the sense that the Euclidean barycenter usually fails to take into account the underlying geometry of these distributions \citep[see e.g.,][]{backhoff2018bayesian}. Thus, despite the potential computational obstacles, we specifically consider the Wasserstein barycenter (and possibly its entropic-regularized variants). 
	We also provide further discussion on the connections of this formulation to other related machine learning paradigms in 	\Cref{sec:add_related}.

	\section{Wasserstein Barycentric Distributionally Robust Optimization}
	\label{sec:Wass_bary_DRO}
	Although the use of the Wasserstein barycenter might give a better consensus representation of samples from different sources, discrepancy between the target distribution and the Wasserstein barycenter of the empirical distributions might still arise, due to e.g., sampling errors and data heterogeneity across the sources. In a similar spirit to the single-source case, we propose to hedge against the impact of such model misspecification through the lens of WDRO. 
	
	\paragraph{Wasserstein Ambiguity Sets with Wasserstein Barycenters.}
	Suppose that we have a finite number of $K$ data sources with probability distributions $\QQ_1,\ldots,\QQ_K$ respectively. 
	Aside from directly considering the notion of Wasserstein barycenters, one way to construct an ambiguity set out of these $K$ probability distributions is to consider the intersection of the individual Wasserstein ambiguity sets $ \bigcap_{k=1}^K \euW_{\varepsilon,p}(\QQ_k) $, but if the data sources are very heterogeneous then this could lead to an overly conservative  ambiguity set using the same large $\varepsilon$. 	
	Alternatively, a less conservative ambiguity set based on these $K$ distributions can be defined by 
	\begin{equation}\label{eqn:ambiguity_alt}
		\teuW_{\mkern -1mu\varepsilon,p} (\QQ_1, \ldots, \QQ_K; \blambda) \\
		\coloneqq \left\{\Prob \in\scrP_p(\Xi) : \sumK\lambda_k\sfW_p^p(\Prob, \QQ_k)\le \varepsilon^p \right\}. 
	\end{equation}
	It is straightforward to observe that $\bigcap_{k=1}^K \euW_{\varepsilon,p}(\QQ_k) \subseteq \teuW_{\mkern -1mu\varepsilon,p}(\QQ_1, \ldots, \QQ_K; \blambda)$ for any $\blambda\in\triangle^K$. 
	In the following, we illustrate that how the ambiguity set \eqref{eqn:ambiguity_alt} is related to another ambiguity set defined with the $\blambda$-weighted $p$-Wasserstein barycenter of $\QQ_1, \ldots, \QQ_K$. 	
	\begin{definition}[Wasserstein barycentric ambiguity set]
		For $\blambda\in\triangle^{K}$, the \emph{$p$-Wasserstein barycentric ambiguity set} with radius $\varepsilon\ge0$ centered at a $\blambda$-weighted $p$-Wasserstein barycenter of $\QQ_1, \ldots, \QQ_K$ (if exists), denoted by $\bar{\QQ}_{\blambda,p}\coloneqq\hat{\sfb}_{\blambda,p}(\QQ_1,\ldots,\QQ_K)$, is defined by 
		\begin{equation}\label{eqn:ambiguity_barycenter}
			\bar{\euW}_{\varepsilon,p}(\QQ_1, \ldots, \QQ_K; \blambda) 
			\coloneqq \left\{\Prob \in\scrP_p(\Xi) : \sfW_p(\Prob, \bar{\QQ}_{\blambda,p})\le \varepsilon \right\}. 
		\end{equation}
	\end{definition}
	Note that $\bar{\euW}_{\varepsilon,p}(\QQ_1, \ldots, \QQ_K; \blambda) = \euW_{\varepsilon, p}(\bar{\QQ}_{\blambda,p})$. 	
	The ambiguity sets \eqref{eqn:ambiguity_alt} and \eqref{eqn:ambiguity_barycenter} of different radii (differed by a factor of $2^p$) can be related by the following inclusion. 
	\begin{theorem}\label{thm:bounds}
		For $\blambda\in\triangle^{K}$, suppose that a $\blambda$-weighted $p$-Wasserstein barycenter of $\QQ_1, \ldots, \QQ_K$ exists. Then, for any $\varepsilon\ge0$, the following inclusion of the ambiguity sets \eqref{eqn:ambiguity_alt} and \eqref{eqn:ambiguity_barycenter} of different radii holds: 
		\begin{equation}\label{eqn:nested_inclusions}
			\teuW_{\mkern -1mu\varepsilon,p}(\QQ_1, \ldots, \QQ_K; \blambda) \subseteq \bar{\euW}_{2^{p}\cdot\varepsilon,p}(\QQ_1, \ldots, \QQ_K; \blambda). 
		\end{equation}
	\end{theorem}
	All proofs of this paper are deferred to \Cref{sec:proofs}. 
	Consequently, it is more feasible to use the Wasserstein barycentric ambiguity set \eqref{eqn:ambiguity_barycenter} since $\varepsilon$ is often tuned in practice. We can therefore apply existing results of WDRO if the Wasserstein barycenter $\bar{\QQ}_{\blambda}$ exists and is available, without the need of redeveloping tools to handle the ambiguity set \eqref{eqn:ambiguity_alt}.

	\begin{remark}
		The overall rationale of the construction of \eqref{eqn:ambiguity_barycenter} is that the most adverse distribution should be close to a population $p$-Wasserstein barycenter $\sfb^\star \coloneqq \sfb_p(\Prob^\star)\in\scrP_p(\Xi)$ of the unknown true distribution (of distributions) $\Prob^\star\in\scrW_p(\scrP_p(\Xi))$ within a radius $\varepsilon$ in $\sfW_p$ distance. This population barycenter can be approximated by its empirical counterpart (see \Cref{subsec:paradigms} for details). This is as opposed to the usual WDRO in which the most adverse distribution is close to the unknown true distribution $\Prob^\star\in\scrP_p(\Xi)$ approximated by a nominal (empirical) distribution $\Probhat$. 
	\end{remark}

	We now revisit the WDRO problem \eqref{eqn:dro} with $K$ available nominal distributions $\Probhat = (\Probhat_1, \ldots, \Probhat_K)$ and the Wasserstein barycentric ambiguity set $\euU_\varepsilon(\Probhat) = \bar{\euW}_{\varepsilon,p}(\Probhat; \blambda)$, which we refer to as the \emph{Wasserstein barycentric DRO} (WBDRO). Consequently, due to the definition of the Wasserstein barycentric ambiguity set \eqref{eqn:ambiguity_barycenter}, solving a WBDRO involves two consecutive steps: (i)~computing a Wasserstein barycenter $\hat{\sfb}_{\blambda,p}(\Probhat_1,\ldots,\Probhat_K)$; (ii)~solving a WDRO problem. Such a problem decomposition enables us to leverage existing theoretical results and computational tools for Wasserstein barycenters \citep[e.g.,][]{alvarez2016fixed,heinemann2022randomized,carlier2015numerical,chewi2020gradient} and WDRO  \citep[e.g.,][]{blanchet2021optimal} respectively.

	In the original WDRO formulation \eqref{eqn:dro}, we usually observe some i.i.d.~realizations $\euD \coloneqq \{z_i\}_{i=1}^n$ of the unknown true distribution $\Prob^\star$, from which we can construct a nominal distribution $\Probhat^n$, e.g., the empirical distribution $\Probhat^n = \frac1n\sumn\updelta_{z_i}$. 
	However, in WBDRO, the way of constructing the $K$ nominal distributions $\Probhat_1, \ldots, \Probhat_K$ as approximations of their corresponding unknown true distributions $\Prob_1^\star, \ldots, \Prob_K^\star$ becomes more subtle. For instance, they can be constructed from the same set of observed data $\euD_n$ via resampling techniques such as bootstrap (requiring $K\le n$), or from $K$ different sources where the observed data at the $k$th source $\euD_{k} = \{z_{i,k}\}_{i=1}^n$ are i.i.d.~realizations of its unknown true distribution $\Prob_k^\star$ (assuming the same sample size $n$ for simplicity). The former scenario is often used to avoid overfitting, while the latter is often encountered in the setting of federated learning with possibly heterogeneous data sources. 
	Again, for each $k\in\setK$, the nominal distributions $\Probhat_k^n$ can be taken as the empirical distributions. Note that $\Probhat_k^n$ converges to $\Prob_k^\star$ in $\sfW_p$ distance as $n\to\infty$ \citep[see e.g.,][]{fournier2015rate,bolley2007quantitative}.

	Under this construction, we are interested in statistical properties of WBDRO in the following two case: (i) $n\to\infty$ with fixed and finite $K$; (ii) $K\to\infty$ with fixed and finite $n$. We argue that both cases are well motivated by the two statistical frameworks with applications discussed in \citet{le2017existence,boissard2015distribution}.

	\subsection{Two Statistical Paradigms}
	\label{subsec:paradigms}		

	\paragraph{Asymptotic Sample Size.}
	Under this paradigm, there are $K$ unknown true distributions $\Prob_1^\star, \ldots, \Prob_K^\star\in\scrP_p(\Xi)$, where $K\in\NN^*$ is finite and fixed. These are approximated by a sequence of distributions constructed from data samples, e.g., the empirical measures $\Probhat_k^n \coloneqq \frac1n\sumn\updelta_{z_{i,k}}$, $k\in\setK$. 
	A crucial result in this case is that, for $\blambda\in\triangle^K$, a $p$-Wasserstein barycenter of $\hat{\rho}_{\blambda}^n \coloneqq \sumK \lambda_k\updelta_{\Probhat_k^n}$ (if exists) converges to a $p$-Wasserstein barycenter of the limit $\rho_{\blambda} \coloneqq \sumK \lambda_k\updelta_{\Probhat_k^\star}$ in $\sfW_p$ distance as $n\to\infty$, by \citet[Theorem 3]{le2017existence}.

	\paragraph{Asymptotic Number of Data Sources.}
	For now, we consider the case where the unknown true distribution $\Prob^\star\in\scrW_p(\scrP_p(\Xi))$ is approximated by a growing discrete distribution $\rho_K$ supported on $K$ elements, with $K\to\infty$. Consider a sequence of $K$ distributions $\Prob_k\in\scrP_p(\Xi)$ with weights $\lambda_k^K\ge0$ for each $k\in\setK$, from which we define the sequence of distributions by $\rho_K \coloneqq \sumK \lambda_k^K\updelta_{\Prob_k}$, where $K\in\NN^*$. 	
	Assume that $\rho_K$ converges to some distribution $\Prob^\star$ in $\sfW_p$ distance. Then the $p$-Wasserstein barycenter of $\rho_K$ converges to the $p$-Wasserstein barycenter of $\Prob^\star$ in $\sfW_p$ distance as $K\to\infty$, by \citet[Theorem 3]{le2017existence}.

	\begin{remark}
		The case of $n$ and $K$ both growing to infinity is even more of interest to our case. Indeed, since $\Probhat_k^n\to\Prob_k^\star$ in $\sfW_p$ distance as $n\to\infty$, for each $k\in\setK$, using the above argument with $\Prob_k$ replaced by $\Prob_k^\star$ for each $k\in\setK$, the $p$-Wasserstein barycenter of $\sumK \lambda_k^K\updelta_{\Prob_k^\star}$ converges to the $p$-Wasserstein barycenter of $\Prob^\star$ in $\sfW_p$ distance as $K\to\infty$. 		
	\end{remark}

	\begin{remark}
		Under some more specific frameworks, e.g., in deformation models  \citep{allassonniere2007towards,allassonniere2013statistical}, the $2$-Wasserstein barycenter of $\Probhat_1, \ldots, \Probhat_K$ is a consistent estimate of $\Prob^\star$, in the sense that $\hat{\sfb}_{\One/K,2}(\Probhat_1, \ldots, \Probhat_K) \to \Prob^\star$ as $K\to\infty$ in $\sfW_2$ distance. 		
		In the case with empirical observations	available, as both $n\to\infty$ and $K\to\infty$, $\hat{\rho}_{n,K} \coloneqq \frac1K\sumK\updelta_{\Probhat_k^n} \to \Prob^\star$ in $\sfW_2$ distance.
		We refer to \citet[Theorem 4.2 and Proposition 5.1]{boissard2015distribution} and \citet{bigot2018characterization,zemel2019frechet} for details. 		
	\end{remark}

	\subsection{Performance Guarantees}
	\label{subsec:performance}
	We now study the implications of the two statistical paradigms in \Cref{subsec:paradigms} on performance guarantees of WBDRO. To simplify discussion, we only consider the case of $p=2$ and equal weights, i.e., $\lambda_k = 1/K$ for each $k\in\setK$. To simplify discussion, we also assume that all $2$-Wasserstein barycenters exist (subject to some technical regularity conditions). In this subsection, with slight abuse of notation, we write $\Probhat = (\Probhat_k)_{k\in\setK}$, $\Probhat^n = (\Probhat_k^n)_{k\in\setK}$, and $\Probhat^\star = (\Probhat_k^\star)_{k\in\setK}$. 

	The finite-sample guarantees of WBDRO are derived by the rates of convergence of Wasserstein barycenters \citep{ahidar2020convergence,schotz2019convergence,le2019fast}, characterized by measure concentration of the $2$-Wasserstein barycenter of the nominal distributions $\Probhat_1^n, \ldots, \Probhat_K^n$. 
	
	With data samples $\euD_{n,k} = \{z_{i,k}\}_{i=1}^n$ for each $k\in\setK$, let $\hat{\rho}_{n,K} \coloneqq \frac1K\sumK\updelta_{\Probhat_k^n} $ where $\Probhat_k^n \coloneqq \frac1n\sumn\updelta_{z_{i,k}}$. For now, we let $K\in\NN^*$ be finite and fixed. The following measure concentration result simplified from \citet[Theorem 12]{le2019fast} states that the $2$-Wasserstein barycenter of sub-Gaussian $\hat{\rho}_K^\star \coloneqq \frac1K\sumK\updelta_{\Probhat_k^\star}$ should be contained in the $2$-Wasserstein barycentric ambiguity set centered at the $2$-Wasserstein barycenter of $\hat{\rho}_{n,K}$ in WBDRO with high probability. 
	\begin{theorem}[Concentration inequality]\label{thm:conc_1}
		Let $K\in\NN^*$ be finite and fixed. 
		Suppose that $\hat{\rho}_K^\star\in\scrW_2(\scrP_2(\Xi))$ is sub-Gaussian 
		with $2$-Wasserstein barycenter $\hat{\sfb}_K^\star\coloneqq \sfb_2(\hat{\rho}_K^\star) =\hat{\sfb}_{\One/K, 2}(\Probhat^\star) \in\scrP_2(\Xi)$. Then $\hat{\sfb}_K^\star$ is unique and there exist constants $(c_1, c_2)\in(0,+\infty)^2$ independent of $n$ such that for any $\beta\in(0,1)$, the concentration inequality 
		\begin{equation*}
			\Prob^{n} \left\{\hat{\sfb}_K^\star \in \bar{\euW}_{\varepsilon,2}(\Probhat^n; \One/K) \right\} 
			\ge 1-\beta-\e^{-c_2n} 
		\end{equation*}
		holds whenever $\varepsilon$ exceeds 	
		\begin{equation}\label{eqn:eps}
			\varepsilon_n(\beta) = \sqrt{\frac{c_1}{n}\log\left(\frac2\beta\right)}. 
		\end{equation}
	\end{theorem}
	\Cref{thm:conc_1} indicates that any $2$-Wasserstein barycentric ambiguity set $\bar{\euW}_\varepsilon(\Probhat^n; \One/K)$ with radius $\varepsilon\ge\varepsilon_n(\beta)$ represents an approximate $(1-\beta)$-confidence region for $\hat{\sfb}_K^\star$, which is a $K$-sample approximation of the $2$-Wasserstein barycenter of the unknown true distribution $\Prob^\star$. 
	
	Note that the distributional uncertainty radius $\varepsilon_n(\beta)$ decays as $\scrO(n^{-\sfrac12})$. Therefore, there is no curse of dimensionality in the uncertainty dimension $m$ when choosing the distributional uncertainty radius $\varepsilon_n(\beta)$, as opposed to the case of WDRO (see \citealp[\S3]{kuhn2019wasserstein} and \citealp[Remark 37]{shafieezadeh2019regularization}).

	From \Cref{thm:conc_1}, we can immediately derive the following finite-sample guarantee. 	
	\begin{theorem}[Finite-sample guarantee]\label{thm:finite_1}
		Suppose that all conditions of \Cref{thm:conc_1} hold with $\varepsilon_{n}(\beta)$ defined in \eqref{eqn:eps}. Then for all $\beta\in(0,1)$ and $\varepsilon\ge\varepsilon_{n}(\beta)$, we have 
		\begin{equation*}
			\Prob^{n}\left\{(\forall\ell\in\euL)\;\euR_{\hat{\sfb}_K^\star}(\ell) \le \euR_{\bar{\euW}_{\varepsilon,2}(\Probhat^n; \One/K)}(\ell) \right\} 
			\ge 1 - \beta-\e^{-c_2n}. 
		\end{equation*}
	\end{theorem}
	\Cref{thm:finite_1} asserts that the worst-case risk provides an upper confidence bound on the approximate true risk under the $2$-Wasserstein barycenter of the unknown true distribution $\Prob^\star$ uniformly across all loss functions $\ell\in\euL$. In particular, if we take $\ell$ to be an optimizer of $\euR_{\bar{\euW}_{\varepsilon,2}(\Probhat^n; \One/K)}$, then this result implies that the optimal value of WBDRO provides an upper confidence bound on the out-of-sample performance of its optimizers.

	When $K$ is fixed and finite, we recall that $\varepsilon_n\to0$ as $n\to\infty$. We can then derive asymptotic consistency of WBDRO in the sample size $n$, which asserts that the solution of WBDRO converges to the worst-case optimal risk under $\hat{\sfb}_K^\star$ with a suitably chosen $\beta=\beta_n\to0$ decaying to $0$ as $n\to\infty$. On the other hand, let us also recall the $2$-Wasserstein barycenter $\hat{\sfb}_{\One/K,2}(\Probhat^\star)$ converges to a $2$-Wasserstein barycenter $\sfb_2(\Prob^\star)$ of the unknown true distribution $\Prob^\star\in\scrW_2(\scrP_2(\Xi))$ in $\sfW_2$ distance as $K\to\infty$. Using this fact, we can also derive the asymptotic consistency of WBDRO in the number of data sources $K$. 
	
	\begin{theorem}[Asymptotic consistency]\label{thm:consistency_1}
		Suppose that all conditions of \Cref{thm:conc_1} hold. If $K\in\NN^*$ is finite and fixed, we can choose $\beta_n\in(0,1)$ and $\varepsilon_n = \varepsilon_n(\beta_n)$ given in \eqref{eqn:eps}, $n\in\NN^*$, satisfying $\sum_{n=1}^\infty \beta_n < \infty$ and $\lim_{n\to\infty}\varepsilon_n(\beta_n)=0$. If $\ell$ is upper semicontinuous and there exists $C>0$ such that $|\ell(\xi)| \le C(1+\norm{\xi}^2)$ for all $\ell\in\euL$ and $\xi\in\Xi$, then we have for $\Prob^\infty$-almost surely, as $n\to\infty$,
		\[\euR_{\bar{\euW}_{\varepsilon_n(\beta_n),2}(\Probhat^n; \One/K)}(\euL) \to \euR_{\hat{\sfb}_K^\star}(\euL).\]  
		Furthermore, suppose that $\Prob^\star$ is sub-Gaussian with $2$-Wasserstein barycenter $\sfb^\star\coloneqq\sfb_2(\Prob^\star)$. If we choose $\gamma_K\in(0,1)$ and $\omega_K = \sqrt{c_3\log\left(2/\gamma_K\right)/K}$ with $c_3>0$ independent of $K$ satisfying $\sum_{K=1}^\infty \gamma_K < \infty$ and $\lim_{K\to\infty}\omega_K(\gamma_K)=0$, then we have for $\Prob^\infty$-almost surely, $\euR_{\hat{\sfb}_K^\star}(\euL) \to \euR_{\sfb^\star}(\euL)$ as $K\to\infty$. 		
	\end{theorem}
	The second part of \Cref{thm:consistency_1} implies that, if $\hat{\sfb}_K^\star$ (or $\Probhat^\star$) is available, the solution of WBDRO converges to the worst-case optimal risk under $\sfb^\star$ with a suitably chosen $\gamma_K$ decaying to $0$ as $K\to\infty$.

	\section{Gelbrich Ambiguity Set with 2-Wasserstein Barycenter}
	\label{sec:Gelbrich}
	In WDRO, more precise structural assumptions about the nominal distribution $\Probhat$ can be made, e.g., it belongs to some family of distributions. In this section, we consider the case of $\Probhat\in\scrP_2(\RR^m)$ with mean $\muhat \in\RR^m$ and covariance matrix $\Sigmahat\in\bbS_+^m$. 	
	While computing the Wasserstein distance between any two distributions is NP-hard in general \citep{taskesen2021semi}, an analytical lower bound of their $\sfW_2$ distance is indeed given by the Gelbrich distance \citep{gelbrich1990formula}, which only involves their first two moments. 
	\begin{definition}[Gelbrich distance]
		For any $\Prob_1, \Prob_2\in\scrP_2(\RR^m)$ with means
		$\mu_1, \mu_2\in\RR^m$ and covariance matrices $\Sigma_1, \Sigma_2\in\bbS_{++}^m$, 
		their \emph{Gelbrich distance} is defined through 
		\begin{multline}
			\sfG((\mu_1, \Sigma_1), (\mu_2, \Sigma_2)) \coloneqq \sqrt{\euclidnorm{\mu_1-\mu_2}^2 + \sfB^2(\Sigma_1, \Sigma_2)}, \\ \text{where}\;\;
			\sfB^2(\Sigma_1, \Sigma_2) \coloneqq \tr(\Sigma_1) + \tr(\Sigma_2) - 2\,\tr\left(\Sigma_1^{\sfrac12} \Sigma_2 \Sigma_1^{\sfrac12}\right)^{\negthickspace\sfrac12}  \label{eqn:Bures}
			\vspace*{-1mm}
		\end{multline}
		is the squared \emph{Bures--Wasserstein distance} \citep{bhatia2019bures}. 
	\end{definition}
	For any $\Prob_1,\Prob_2\in\scrP_2(\RR^m)$, the \emph{Gelbrich bound} \citep{gelbrich1990formula}: 
	\[\sfW_2(\Prob_1, \Prob_2) \ge \sfG((\mu_1, \Sigma_1), (\mu_2, \Sigma_2))\] 
	holds, where equality holds if $\Prob_1, \Prob_2$ belongs to the same location-scatter family $\euF(\Prob_0)$ for some $\Prob_0\in\scrP_2^{\ac}(\RR^m)$, which is defined as follows. 	
	\begin{definition}[Location-scatter family]
		Let $X_0\in\RR^m$ be a random vector with $\Law(X_0)= \Prob_0\in\scrP_2^{\ac}(\RR^m)$. The set 		
		$\euF(\Prob_0) \coloneqq \{\Law(AX_0 + b) : A\in\bbS_{++}^m, b\in\RR^m\}$
		of probability distributions induced by positive definite affine transformations from $\Prob_0$ is called a \emph{location-scatter family}. Notably, location-scatter families encompass the Gaussian and elliptical distributions \citep[see e.g.,][]{muzellec2018generalizing}. 
	\end{definition}
	We then define the \emph{mean-covariance ambiguity set} \citep{kuhn2019wasserstein,nguyen2021mean} as a ball centered at $(\muhat, \Sigmahat)$ with radius $\varepsilon\ge0$ in terms of the Gelbrich distance by $\euV_\varepsilon(\muhat, \Sigmahat) \coloneqq \{(\mu, \Sigma) \in\RR^m\times\bbS_{+}^m : \sfG((\mu,\Sigma), (\muhat, \Sigmahat)) \le \varepsilon\}$.
	Next, we define the \emph{Gelbrich ambiguity set} \citep{kuhn2019wasserstein,nguyen2021bridging,nguyen2021mean}, which is the preimage of $\euV_\varepsilon(\muhat, \Sigmahat)$ under the mean-covariance projection, through 
	\begin{equation}\label{eqn:Gelbrich_ambiguity}
		\euG_\varepsilon(\muhat, \Sigmahat) \coloneqq \left\{ \QQ\in\scrP_2(\Xi) : 
		\left( \Ex_{\QQ}[\xi], \Var_{\QQ}(\xi) \right) \in \euV_\varepsilon(\muhat, \Sigmahat) \right\}. 
	\end{equation}

	Now we consider the setting of \Cref{sec:Wass_bary_DRO} with $p=2$, and that all $\QQ_k$'s belong to the same location-scatter family. The $2$-Wasserstein barycenter $\bar{\QQ}_{\blambda,2}$ also belongs to the same location-scatter family \citep[see \Cref{prop:location-scatter} for details]{alvarez2018wide}. 
	
	Note that we can recover the Gelbrich ambiguity set centered at $\bar{\QQ}_{\blambda,2}$ if we impose further distributional restrictions on the $2$-Wasserstein barycentric ambiguity set. 
	\begin{proposition}\label{prop:location-scatter_ambiguity}
		Let $\Prob_0\in\scrP_2^{\ac}(\RR^m)$ and $\QQ_1, \ldots, \QQ_K\in\euF(\Prob_0)$ with means $\mu_1, \ldots,\mu_K\in\RR^m$ and covariance matrices $\Sigma_1, \ldots,\Sigma_K\in\bbS_{++}^m$ respectively. 
		For $\blambda\in\triangle^K$, the $\blambda$-weighted $2$-Wasserstein barycenter of $\QQ_1, \ldots, \QQ_K$ is given by $\bar{\QQ}_{\blambda,2}\in\euF(\Prob_0)$ with mean $\bar{\mu}_{\blambda}\in\RR^m$ and covariance matrix $\bar{\Sigma}_{\blambda}\in\bbS_{++}^m$. If the family of distributions in the $2$-Wasserstein barycentric ambiguity set also belongs to $\euF(\Prob_0)$, i.e, 
		$\bar{\euW}_{\varepsilon, 2}^{\mathsf{ls}}(\QQ_1, \ldots, \QQ_K; \blambda) 
		\coloneqq  \left\{\Prob\in \euF(\Prob_0) : \sfW_2(\Prob, \bar{\QQ}_{\blambda,2}) \le \varepsilon \right\}$, 
		then this ambiguity set equals the Gelbrich ambiguity set \eqref{eqn:Gelbrich_ambiguity} centered at $(\mubar_{\blambda}, \Sigmabar_{\blambda})$ restricted to $\euF(\Prob_0)$, i.e., 
		\begin{equation*}
			\bar{\euW}_{\varepsilon, 2}^{\mathsf{ls}}(\QQ_1, \ldots, \QQ_K; \blambda) = \euG_\varepsilon(\mubar_{\blambda}, \Sigmabar_{\blambda}) \cap \euF(\Prob_0).
		\end{equation*}
	\end{proposition}

	Similar to WDRO (cf.~\Cref{coro:risk_bounds}), we can derive the following (optimal) worst-case risk upper bounds for $2$-WBDRO with the \emph{Gelbrich risk} (i.e., risk under the Gelbrich ambiguity set).
	\begin{theorem}\label{thm:bary_risk_bounds}
		Assume that $\Probhat_k\in\scrP_2^{\ac}(\Xi)$ has mean $\mu_k\in\RR^m$ and covariance matrix $\Sigma_k\in\bbS_{++}^m$ for each $k\in\setK$. 
		Then, we have 
		\begin{equation*}
			(\forall\ell\in\euL)\quad \euR_{\bar{\euW}_{\varepsilon,2}(\Probhat_1, \ldots, \Probhat_K; \blambda)}(\ell) \le \euR_{\euG_\varepsilon(\mubar_{\blambda}, \Sigmabar_{\blambda})}(\ell) \quad \text{and} \quad  \euR_{\bar{\euW}_{\varepsilon,2}(\Probhat_1, \ldots, \Probhat_K; \blambda)}(\euL) \le \euR_{\euG_\varepsilon(\mubar_{\blambda}, \Sigmabar_{\blambda})}(\euL), 
		\end{equation*}		
		where $\mubar_{\blambda}$ and $\Sigmabar_{\blambda}$ are the mean and covariance matrix of  $\hat{\sfb}_{\blambda,2}(\Probhat_1,\ldots,\Probhat_K)$ respectively. 
	\end{theorem}
	The above risk bounds indeed reveal a trade-off between tractability and the use of available information. While the worst-case Gelbrich risk minimization problem is more tractable, it uses merely information of the nominal distributions up to their first two moments and discards higher-order moment information.

	\section{Distributionally Robust Inverse Covariance Matrix Estimation}
	\label{sec:DRMLE}
	We demonstrate the proposed WBDRO via an example of sparse inverse covariance (precision) matrix estimation with a Wasserstein barycentric ambiguity set for a Gaussian random vector $\xi\in\RR^m$ with covariance matrix $\Sigma\in\bbS_{++}^m$, where $n$ independent samples are obeserved for each of the $K$ possibly heterogeneous data sources. 	
	Estimation of precision matrices is of more interest than that of covariance matrices since it finds various applications to, e.g., mean-variance portfolio optimization and linear discriminant analysis. However, the sample covariance matrix $\Sigmahat$ is usually rank-deficient when $m> n$ even if $\Sigma$ has full rank, so na\"{i}vely inverting $\Sigmahat$ to obtain a meaningful precision matrix estimator is not viable. 
	
	For simplicity, we assume that the unknown true distribution $\Prob^\star$ has zero mean. In this case the precision matrix is usually estimated via maximum likelihood estimation (MLE) by minimizing 
	\[f(X) \coloneqq -\log\det X + \frac1n \sumn\dotp{z_i}{Xz_i} \] 
	over $\bbS_{++}^m$ with $n$ independent samples $\{z_i\}_{i=1}^n$. However, this MLE problem is unbounded for $n\le m$. \citet{nguyen2022distributionally} alleviate this issue by incorporating distributional robustness using a $2$-Wasserstein ambiguity set centered at the nominal distribution $\euN(0, \Sigmahat)$, leading to the \emph{Wasserstein Shrinkage Estimator} (WSE), which can be solved in a quasi-closed form. 
	
	With observed data from $K$ data sources $\{z_{i,k}\}_{i\in\setn, k\in\setK}$, it is unclear how to construct a common estimator using aggregate information from them even in the low-dimensional regime using the MLE approach other than a simple weighted average. In view of this, we propose the use of WBDRO to construct a distributionally robust aggregate estimator. We consider the Bures--Wasserstein ambiguity set centered at the $\blambda$-weighted $2$-Wasserstein barycenter $\bar{\Sigma}_{\blambda}$ of $\Probhat_1^n, \ldots, \Probhat_K^n$, where $\Probhat_k^n = \euN(0, \Sigmahat_k^n)$ with empirical covariance $\Sigmahat_k^n = \frac1n\sumn z_{i,k}z_{i,k}^\top$ for each $k\in\setK$. 
	The \emph{Bures--Wasserstein ambiguity set} is defined by 
	\[\euB_\varepsilon(\Sigmahat) \coloneqq \{\QQ\sim\euN(0,\Sigma) : \sfB(\Sigma, \Sigmahat) \le \varepsilon\}. \]
	Note that $\bar{\Sigma}_{\blambda}$ also coincides the $\blambda$-weighted Bures--Wasserstein barycenter \citep{kroshnin2021statistical} of $\Sigmahat_1, \ldots, \Sigmahat_K$ (see \Cref{subsec:add_Gelbrich}). 	
	The distributionally robust maximum likelihood estimation (DRMLE) problem can hence be formulated as 
	\begin{equation}\label{eqn:DRMLE}
		\minimize_{X\in\bbS_+^m} \ \left\{-\log\det X + \sup_{\Prob\in\euB_\varepsilon(\bar{\Sigma}_{\blambda})}\Ex_{\xi\sim\Prob}[\dotp{\xi}{X\xi}]\right\}. 
	\end{equation}
	Note that the population Wasserstein barycenter of $\Prob^\star\in\scrW_2(\scrP_2(\RR^m))$ is also Gaussian since any location-scatter family (which includes Gaussian) is closed for Wasserstein barycenters \citep{alvarez2018wide}. 	
	The problem \eqref{eqn:DRMLE} is indeed equivalent to a WDRO problem with a Wasserstein ambiguity set centered at the Wasserstein barycenter $\euN(0, \bar{\Sigma}_{\blambda})$, which also admits an analytical solution and is referred to as the \emph{Wasserstein Barycentric Shrinkage Estimator} (WBSE). Further details are given in \Cref{sec:add_DRMLE}.

	\paragraph{Simulations.}
	We compare WBSE with two other estimators constructed from widely used precision matrix estimators for single data source, namely linear shrinkage (LS) and $L_1$-regularized maximum likelihood estimators ($L_1$). We choose $m=20$, $\blambda=\One/K$, $n\in\{50,100, 200\}$, $K\in\{25, 50, 100\}$, in order to observe the effects of both the sample size $n$ and the number of data sources $K$ on each estimator. We generate $K$ sparse matrices in $\bbS_{++}^m$ with sparsity level $s = 50\%$ as the true precision matrices $\Sigma_k^{-1}$. The true covariance matrix $\Sigma^\star$ is approximated by the Bures--Wasserstein barycenter of another 1000 samples of $\Sigma_{k'}$. Then the true precision matrix is $X^\star = (\Sigma^\star)^{-1}$. We then generate samples $\{z_{i,k}\}_{i=1}^n$ from each of $\euN(0, \Sigma_k)$ to construct the empirical covariance matrices $\Sigmahat_k$, which are used to compute $\bar{\Sigma}_{\blambda}$ and construct the three estimators. See \Cref{sec:add_DRMLE} for additional details. 
	
	We measure performance of estimators using the Stein loss $L(\hat{X}, \Sigma^\star) \coloneqq -\log\det(\hat{X}\Sigma^\star) + \tr(\hat{X}^\top\Sigma^\star) - m$, which vanishes if $\hat{X} = (\Sigma^\star)^{-1}$. The losses of the estimators are given in \Cref{table}, averaged over $20$ independent trials. We observe that WBSE outperforms the other two estimators by a large margin. The performance of WBSE also improves as $n$ and $K$ increase.

	\begin{table}[h]
		\caption{Stein losses of LS, $L_1$ and WBSE.}
		\label{table}
		\begin{center}
			\begin{small}
				\begin{tabular}{ccccc}
					\toprule
					$n$ & $K$ & LS & $L_1$ & WBSE  \\
					\midrule						  
					& 25 &  6.77 $\pm$ 0.58  &  7.66 $\pm$ 0.63  &  1.77 $\pm$ 0.30  \\						 
					50 & 50 &  6.81 $\pm$ 0.43  &  7.72 $\pm$ 0.46  &  1.27 $\pm$ 0.20  \\
					& 100	&   6.91 $\pm$ 0.29 &  7.84 $\pm$ 0.31  & 0.99 $\pm$ 0.14  \\
					\midrule
					& 25 & 6.72 $\pm$ 0.49 &  7.61 $\pm$ 0.53   &   1.74 $\pm$ 0.27 \\
					100 & 50 & 6.76 $\pm$ 0.33  &  7.67 $\pm$ 0.35  & 1.32 $\pm$ 0.19 \\					
					& 100 &  6.90 $\pm$ 0.27 &  7.83 $\pm$ 0.29  & 1.12 $\pm$ 0.16 \\ 
					\midrule
					& 25	&  6.68 $\pm$ 0.56  &  7.57 $\pm$ 0.60  & 1.76 $\pm$ 0.28 \\
					200 & 50	& 6.72 $\pm$ 0.36   &  7.63 $\pm$ 0.38 & 1.34 $\pm$ 0.21  \\ 
					& 100 &  6.87 $\pm$ 0.29  &  7.79 $\pm$ 0.31  &  0.62 $\pm$ 0.18 \\
					\bottomrule
				\end{tabular}
			\end{small}
		\end{center}
	\end{table}	
	Despite being motivated by the high-dimensional setting, the proposed WBS estimator is not feasible when $n< m$ since all $\Sigmahat_k$'s are singular and their Bures--Wasserstein barycenter $\bar{\Sigma}_{\blambda}$ does not exist, as opposed to the applicability of WSE. 
	A possible remedy is to consider the entropic-regularized variants of the barycenter \citep{carlier2021entropic,bigot2019penalization,janati2020entropic,mallasto2021entropy,minh2022entropic}. We use the Sinkhorn barycenter (see \Cref{sec:add_details}) and give additional simulation results under the high-dimensional setting in \Cref{sec:add_DRMLE}.

	\section{Concluding Remarks}
	\label{sec:conclusion}
	In this paper, we propose the use of Wasserstein barycenter in the construction of ambiguity sets in WDRO to aggregate data samples from multiple sources. In addition to the performance guarantees established in this paper, extending the 
	statistical analysis \citep{blanchet2021statistical,blanchet2021sample,blanchet2021confidence,bartl2021sensitivity} and generalization bounds \citep{an2021generalization} for WDRO to WBDRO are also important research directions. Motivated by computational tractability and different use cases, alternative barycenters based on other optimal transport distances can also be considered \citep{bonneel2015sliced,carlier2021entropic,janati2020debiased,li2020continuous,kim2018canonical,bigot2019penalization,peyre2016gromov,friesecke2021barycenters,cazelles2021novel}. The same applies to the choice of discrepancy between probability distributions in the ambiguity set in WBDRO \citep{wang2021sinkhorn,azizian2022regularization}. Finally, it is also interesting to build more general machine learning applications upon the general framework of WBDRO.

	\bibliographystyle{plainnat}
	\bibliography{dro_ref}

	\newpage
	\appendix
	\numberwithin{equation}{section}
	\numberwithin{figure}{section}

	\begin{center}
		{\LARGE \textsc{Appendix}}
	\end{center}

	\section{Other Related Work}
	\label{sec:add_related}
	In this section, we provide a more detailed discussion on the connections of the proposed framework to other existing machine learning paradigms. 
	
	We now introduce some additional notation. In the following, we consider a supervised learning setting with the input space $\euX$ and the output space $\euY$. With $n$ samples $\euD_n =\{(x_i, y_i)\}_{i=1}^n$, the most notable framework for building machine learning models is the ERM formulation, which solves 
	\begin{equation*}
		\minimize_{\theta\in\Theta} \ \Ex_{(x,y)\sim\Probhat}[\ell(h_\theta(x), y)] = \frac1n\sumn \ell(h_\theta(x_i), y_i), 
	\end{equation*}
	where $h_\theta\colon\euX\to\euY$ represents the model with parameter $\theta\in\Theta\subseteq\RR^d$, $\ell(\cdot,\cdot)$ is the loss function, and $\Probhat\coloneqq \frac1n\sumn\updelta_{(x_i,y_i)}$ is the empirical distribution of $\euD_n$. 
	
	\paragraph{Federated Learning.}
	Following the discussion in \Cref{sec:learning_multiple_sources} of the main text, federated learning (FL) and its optimization formulation is a crucial motivation of the problem considered in this paper. The agnostic federated learning (AFL) framework \citep{mohri2019agnostic,ro2021communication} considers the worst-case setting in which the learner seeks a solution that is favorable for any $\blambda\in\Lambda\subseteq\triangle^K$, where $\Lambda$ is a closed convex set. Again, if we define the $\blambda$-mixture of distributions $\Prob_{\blambda} \coloneqq \sumK\lambda_k\Prob_k$, then the \emph{agnostic risk} is given by 
	\begin{equation*}
		\euL_{\Prob_\Lambda}(\theta) \coloneqq \sup_{\blambda\in\Lambda} \ \Ex_{(x,y)\sim\Prob_{\blambda}}[\ell(h_\theta(x), y)]. 
	\end{equation*}
	In practice, only the empirical distributions $\Probhat_k$'s are accessible (constructed from finite samples), so we can define the $\blambda$-mixture of empirical distributions $\bar\Prob_{\blambda} \coloneqq \sumK\lambda_k\Probhat_k$. Then, the \emph{agnostic empirical risk} is given by 
	\begin{equation}\label{eqn:afl_emp}
		\euL_{\bar\Prob_\Lambda}(\theta) \coloneqq \sup_{\blambda\in\Lambda} \ \Ex_{(x,y)\sim\bar\Prob_{\blambda}}[\ell(h_\theta(x), y)]. 
	\end{equation}
	Two notable differences between AFL and our proposed framework are that in AFL the choice of $\blambda$ is also optimized, and the use of mixture distributions instead of Wasserstein barycenters. Following this line of work, \citet{deng2020distributionally} develop communication-efficient distributed algorithms for minimizing the agnostic empirical risk \eqref{eqn:afl_emp}, whereas \citet{reisizadeh2020robust} study the notion of robustness against affine distribution drifts in clients' data in federated learning.

	On the other hand, Wasserstein distributionally robust federated learning \citep[WAFL;][]{le2022on} shares a very similar spirit to our work which considers a WDRO formulation under the federated learning setting, but again with the mixture distribution (Euclidean barycenter) usually considered in FL instead of the notion of (entropic-regularized) Wasserstein barycenters used in this work. 
	
	\begin{remark}
		In this work, we however do not aim at solving the FL problem, which requires distributed and decentralized computations. Yet, it is very interesting to make our proposed paradigm amenable to the full federated learning setting, which might require decentralized distributed computation of Wasserstein barycenters \citep{dvurechenskii2018decentralize}. 
	\end{remark}

	\paragraph{(Multiple-Source) Domain Adaptation.}
	In the multiple-source domain adaptation problem \citep[see e.g.,][and references therein]{mansour2021theory,zhang2021multiple}, each domain is defined by the corresponding distribution $\Prob_k$. 
	The target distribution could be assumed to be close to some convex combination of source distributions, $\sumK\lambda_k\Prob_k$. The learner wants to learn a model on the target domain. Similar to the AFL formulation, the learner can do this by solving 
	\begin{equation*}
		\minimize_{\theta\in\Theta} \ \Ex_{(x,y)\sim\bar\Prob_{\blambda^\star}}[\ell(h_\theta(x), y)], 
	\end{equation*}
	where $\blambda^\star \coloneqq \argmin_{\blambda\in\triangle^K} \euE(\blambda)$ for some discrepancy measure $\euE$ between the empirical target distribution $\Probhat_0$ and $\bar\Prob_{\blambda}$, e.g., a Bregman divergence $B(\Probhat_0\midd \bar\Prob_{\blambda})$. Another related work \citet{taskesen2021sequential} study a distributionally robust formulation for supervised domain adaptation with scarce labeled target data.

	\paragraph{Boosting.}
	Also inspired by the agnostic loss in the AFL framework \citep{mohri2019agnostic}, \citet{cortes2021boosting} study boosting in the presence of multiple source domains. They put forward the so-called $\sfQ$-ensembles, which are convex combinations weighted by a domain classifier $\sfQ$. The typical assumption that the target distribution is a mixture of the source distributions is also made. They also provide an algorithmic extension to the federated learning scenario. Further related work can be found in \citet{cortes2021boosting}.

	\paragraph{Group DRO.}
	Machine learning models might rely on \emph{spurious correlations}---misleading heuristics which hold for most training examples but are wrongly linked to the target. Thus, these models could suffer high risk on minority groups where these correlations do not hold. 
	The Group DRO framework \citep{hu2018does,sagawa2019distributionally}, which aims to obtain high performance across all groups, minimizes the \emph{worst-group risk}: 
	\begin{equation}\label{eqn:group_DRO}
		\minimize_{\theta\in\Theta} \sup_{\Prob\in\euQ(\Prob_1,\ldots,\Prob_K;\blambda)} \ \Ex_{\xi\sim\Prob} \left[\ell(h_\theta(x), y)\right], 
	\end{equation}
	where the ambiguity set is defined as $\euQ(\Prob_1,\ldots,\Prob_K;\blambda) \coloneqq \left\{\sumK\lambda_k\Prob_k : \blambda\in\triangle^K\right\}$. 
	
	Note that \eqref{eqn:group_DRO} is equivalent to 
	\[\minimize_{\theta\in\Theta} \sup_{k\in\setK} \ \Ex_{\xi\sim\Prob_k} \left[\ell(h_\theta(x), y)\right]. \]
	Since in practice we only observe the empirical distributions $\Probhat_k$'s, we instead minimize the \emph{empirical worst-group risk}:
	\[\minimize_{\theta\in\Theta} \sup_{k\in\setK} \ \Ex_{\xi\sim\Probhat_k} \left[\ell(h_\theta(x), y)\right], \]
	which can be rewritten as 
	\[
	\minimize_{\theta\in\Theta} \sup_{\blambda\in\triangle^K} \ \left\{\sumK\lambda_k\Ex_{\xi\sim\Probhat_k} \left[\ell(h_\theta(x), y)\right] = \Ex_{\xi\sim\bar\Prob_{\blambda}} \left[\ell(h_\theta(x), y)\right]\right\}.
	\]	
	This has an almost identical formulation to the AFL framework \eqref{eqn:afl_emp} above (when $\Lambda = \triangle^K$). 	
	A recent work \citet{carmon2022distributionally} study an accelerated optimization method for Group DRO. 	
	
	As a work close to Group DRO, \citet{slowik2022distributionally} study the relation between solving a DRO problem and optimizing the expected error for a single distribution constructed by the mixture distribution $\sumK\lambda_k\Prob_k$ for some $\blambda\in\triangle^K$, particularly with nonconvex loss functions.

	\section{Additional Technical Details and Results}
	\label{sec:add_details}
	
	\subsection{Additional Notation and Definitions}
	For any two matrices $A \coloneqq (a_{i,j})_{i\in\set{d_1}, j\in\set{d_2}}$ and $B\coloneqq (b_{i,j})_{i\in\set{d_1}, j\in\set{d_2}}$ in $\RR^{d_1\times d_2}$, we denote the Frobenius inner product of $A$ and $B$ by $\dotpF{A}{B} \coloneqq \tr(A^\top B) = \sum_{i=1}^{d_1}\sum_{j=1}^{d_2} a_{i,j}b_{i,j}$ and the Frobenius norm of $A$ by $\fronorm{A} \coloneqq \sqrt{\dotpF{A}{A}}$. 
	The elementwise $L_1$-norm of $A$ is denoted by $\onenorm{A} \coloneqq \sum_{i=1}^{d_1}\sum_{j=1}^{d_2}  |a_{i,j}|$. 	
	The Lebesgue measure over $\euX$ is denoted by $\calL_\euX$. 
	For two probability measures $\mu$ and $\nu$ on $\euB(\euX)$, the relative entropy or the Kullback--Leibler (KL) divergence from $\mu$ to $\nu$ is $\sfD_{\KL}(\mu\midd\nu) \coloneqq \int_{\euX} \log(\diff\mu/\diff\nu)  \,\diff\mu$ if $\mu$ is absolutely continuous w.r.t.~$\nu$ (denoted by $\mu\ll\nu$) with the Radon--Nikodym derivative $\diff\mu/\diff\nu$ and $+\infty$ otherwise. 	
	The product measure $\mu\otimes\nu\in\scrP(\RR^d\times\RR^d)$ of $\mu$ and $\nu$ is characterized by $(\mu\otimes\nu)(\euX\times\euY) = \mu(\euX)\nu(\euY)$ for any pair of Borel sets $\euX, \euY\subset\RR^d$.

	\begin{definition}[Sub-Gaussian measure]
		Let $\Xi\subseteq\RR^m$ be a closed convex set. 
		A probability measure $\Prob\in\scrW_p(\scrP_p(\Xi))$ is sub-Gaussian with variance proxy $\sigma^2>0$ if 
		\[\Ex_{\rho\sim\Prob}\left[\exp\left\{\frac{1}{2\sigma^2}\sfW_p^2\left( \sfb_p(\Prob), \rho\right) \right\}\right] \le 2, \]
		where $\rho\in\scrP_p(\Xi)$ is a random measure with distribution $\Prob$. 		
	\end{definition}

	\subsection{Entropic Optimal Transport}
	Based on computational consideration, entropic regularization has been introduced to approximate Wasserstein distances. The so-called \emph{entropic(-regularized) optimal transport} has aroused much theoretical and computational interests across the fields of machine learning, statistics, economics, image processing, and theoretical and applied probability. We refer to \citet{nutz2021introduction,nutz2021entropic,ghosal2021stability,bernton2021entropic,goldfield2022statistical,bigot2019central} for recent theoretical advances in probability and statistics.

	In the machine learning community, \citet{cuturi2013sinkhorn} proposes the so-called \emph{Sinkhorn distance} (which is indeed not a metric), which is referred to as the \emph{entropic-$p$-Wasserstein distance} in this paper and is the central object in entropic optimal transport. Let us recall that $\Omega\subseteq\RR^m$ is a closed convex set. 
	\begin{definition}[Entropic-Wasserstein distance]
		For $\sigma>0$, the \emph{entropic-$p$-Wasserstein distance} between $\rho, \nu\in\scrP(\Omega)$ is defined by 
		\begin{equation}\label{eqn:ent_Wass}
			\OT_{p, \sigma}^{\omega_1, \omega_2}(\rho,\nu)  \coloneqq \inf_{\pi\in\Pi(\rho, \nu)} \, \left\{\int_{\Omega\times\Omega}  \norm{x - y}^p\,\diff\pi(x,y) + \sigma \sfD_{\KL}(\pi \midd \omega_1\otimes \omega_2) \right\}, 
		\end{equation}
		where $\omega_1$ and $\omega_2$ are two reference measures such that $\rho\ll \omega_1$ and $\nu\ll \omega_2$. 
		Note that the entropic-$p$-Wasserstein distance equals the $p$-Wasserstein distance when $\sigma\to0$. 
	\end{definition}	
	The choice of the reference measures in \eqref{eqn:ent_Wass} is known to induce different types of entropy bias \citep{janati2020debiased}, since in general $\OT_{p,\sigma}^{\omega_1, \omega_2}(\rho,\rho) \ne 0$ due to the regularization term. For example, the Lebesgue measure ($\omega_1 = \omega_2 = \calL_{\RR^m}$) induces a blurring bias, whereas simply taking the product measure with $\omega_1 = \rho$ and $\omega_2 = \nu$ induces a shrinking bias. 	
	To circumvent the entropy bias, the \emph{$p$-Sinkhorn divergence} \citep{genevay2018learning,feydy2019interpolating,chizat2020faster,luise2019sinkhorn,ramdas2017wasserstein} can be defined without specifying any reference measures \citep{feydy2019interpolating}: 
	\begin{equation}\label{eqn:Sinkhorn_div}
		\sfS_{p,\sigma}(\rho,\nu) \coloneqq \OT_{p,\sigma}(\rho,\nu) - \frac{\OT_{p,\sigma}(\rho,\rho) + \OT_{p,\sigma}(\nu,\nu)}2.  
	\end{equation}
	The Sinkhorn divergence can be viewed as an interpolation between the (unregularized) Wasserstein distance (when $\sigma\to0$) and \emph{maximum mean discrepancy} (MMD; when $\sigma\to\infty$) \citep{ramdas2017wasserstein,feydy2019interpolating}. Note that there have also been lots of recent results regarding the computational efficiency guarantees of entropic optimal transport, see e.g., \citet{chizat2020faster,genevay2019sample}.

	To avoid confusion, we reserve \emph{Sinkhorn} to solely refer to notions defined via the Sinkhorn divergence \eqref{eqn:Sinkhorn_div}, whereas \emph{entropic-Wasserstein} to solely refer to notions defined via the entropic-Wasserstein distance \eqref{eqn:ent_Wass}.

	\paragraph{The Gaussian Case.}
	Similar to the (unregularized) Wasserstein distance, entropic-Wasserstein distances and Sinkhorn divergences usually do not admit closed forms, with the notable exception for the one between two multivariate Gaussians. The following closed form expressions of entropic-$2$-Wasserstein distance and Sinkhorn divergence between two multivariate Gaussians are directly stated from \citet{janati2020entropic,mallasto2021entropy,minh2022entropic} without proofs. 
	
	\begin{proposition}	
		The entropic-$2$-Wasserstein distance between two Gaussians $\rho_k = \euN(\mu_k, \Sigma_k)$ with $\mu_k\in\RR^m$ and $\Sigma_k\in\bbS_{+}^m$ for $k\in\set{2}$ is given by 
		\begin{equation*}
			\OT_{2,\gamma}^\otimes(\rho_1, \rho_2) = \euclidnorm{\mu_1 - \mu_2}^2 + \tr\left( \Sigma_1 + \Sigma_2 - 2D_\gamma^{\Sigma_1,\Sigma_2}\right)  
			+ \frac\gamma2 \left[m(1-\log\gamma) + \log\det(2D_\gamma^{\Sigma_1,\Sigma_2} + \gamma I/2)\right], 
		\end{equation*}
		where $D_\gamma^{\Sigma_1,\Sigma_2} \coloneqq (\Sigma_1^{\half}\Sigma_2\Sigma_1^{\half} + \gamma^2 I/16 )^{\half}$. 	
		Then using \eqref{eqn:Sinkhorn_div}, the Sinkhorn divergence between two Gaussians $\rho_k = \euN(\mu_k, \Sigma_k)$ for $k\in\set{2}$ is given by 
		\begin{multline*}
			\sfS_{2,\gamma}(\rho_1, \rho_2) = \euclidnorm{\mu_1 - \mu_2}^2 + \tr\left( D_\gamma^{\Sigma_1,\Sigma_1} + D_\gamma^{\Sigma_2,\Sigma_2} - 2D_\gamma^{\Sigma_1,\Sigma_2}\right) +\frac{\gamma}{2}\log\det(2D_\gamma^{\Sigma_1,\Sigma_2} 		+ \gamma I/2) \\
			- \frac\gamma4 \left[\log\det(2D_\gamma^{\Sigma_1,\Sigma_1} + \gamma I/2) + \log\det(2D_\gamma^{\Sigma_2,\Sigma_2} + \gamma I/2) \right]. 
		\end{multline*}	
	\end{proposition}
	Note that, unlike the (squared) $2$-Wasserstein distance between two Gaussians (i.e., Gelbrich distance), the entropic-$2$-Wasserstein distance and the Sinkhorn divergence are both defined for degenerate Gaussians, i.e., when (any of or both) $\Sigma_1$ and $\Sigma_2$ are singular.

	\subsubsection{Entropic-Wasserstein and Sinkhorn Barycenters}
	Similar to Wasserstein distances, Wasserstein barycenters are also NP-hard to compute in general \citep{altschuler2022wasserstein}. A potential remedy is to introduce entropic regularization, which we refer to as the entropic-Wasserstein barycenters \citep{carlier2021entropic,bigot2019penalization,janati2020debiased,bigot2019data,kim2018canonical,cuturi2014fast,cuturi2018semidual}. 
	Recent theoretical results on entropic-Wasserstein barycenters such as their existence and uniqueness can be found in \citet{carlier2021entropic}. 	
	A variant of the entropic-Wasserstein barycenter is the Sinkhorn barycenter \citep{janati2020debiased,luise2019sinkhorn}, defined via the Sinkhorn divergence \eqref{eqn:Sinkhorn_div}, in order to debias the entropic-Wasserstein barycenter due to the entropic regularization. 
	We give the definitions of both the empirical entropic-Wasserstein and Sinkhorn barycenters below.

	\begin{definition}[Empirical entropic-Wasserstein and Sinkhorn barycenters]
		For $p\in[1,+\infty)$ and $\blambda=(\lambda_k)_{k\in\setK}\in\triangle^K$, the \emph{$\blambda$-weighted entropic-$p$-Wasserstein barycenter} is 
		\begin{equation*}
			\hat\sfb_{\blambda, p}^{\OT,\sigma}(\rho_1,\ldots,\rho_K) \coloneqq \argmin_{\nu\in\scrP(\RR^m)} \, \sumK \lambda_k\OT_{p,\sigma}^\otimes(\nu, \rho_k), 
		\end{equation*}
		where $\OT_{p,\sigma}^\otimes(\rho,\nu) \equiv \OT_{p,\sigma}^{\rho,\nu}(\rho,\nu)$. 		
		Likewise, for $p\in[1,+\infty)$ and $\blambda=(\lambda_k)_{k\in\setK}\in\triangle^K$, the \emph{$\blambda$-weighted $p$-Sinkhorn barycenter} is 		
		\begin{equation*}
			\hat\sfb_{\blambda, p}^{\sfS,\sigma}(\rho_1,\ldots,\rho_K) \coloneqq \argmin_{\nu\in\scrP(\RR^m)} \, \sumK\lambda_k \sfS_{p,\sigma}(\nu, \rho_k). 
		\end{equation*}		
	\end{definition}

	\paragraph{The Gaussian Case.}
	Unlike the unregularized case, the entropic-$2$-Wasserstein barycenter and Sinkhorn barycenter of Gaussians are no longer guaranteed to be Gaussian, so we have to restrict them to the manifold of Gaussians. Then, under this assumption, similar to the unregularized case, both the entropic-$2$-Wasserstein barycenter and the Sinkhorn barycenter can be computed by solving fixed-point equations \citep{minh2022entropic,janati2020entropic,mallasto2021entropy}. 
	For completeness, we state (without proof) the most general results from \citet[Theorems 11--12]{minh2022entropic} below, and refer the readers to \citet{minh2022entropic} for details. 
	
	\begin{proposition}\label{prop:ent_barycenter_Gaussian}
		Let $\rho_1,\ldots,\rho_K$ be $K$ possibly degenerate Gaussian distributions $\rho_k=\euN(\mu_k,\Sigma_k)$ with $\mu_k\in\RR^m$ and $\Sigma_k\in\bbS_{+}^m$ for $k\in\setK$. Their entropic-$2$-Wasserstein barycenter restricted to the manifold of Gaussians is $\hat\sfb_{\blambda, p}^{\OT,\sigma}(\rho_1,\ldots,\rho_K)=\euN(\mubar_{\blambda}, \Sigmabar_{\blambda,\sigma})$, where $\mubar_{\blambda} = \sumK\lambda_k\mu_k$ and $\Sigmabar_{\blambda,\sigma}$ satisfies the equation 
		\begin{equation*}
			\Sigmabar_{\blambda,\sigma} = \frac{\sigma}{4}\sumK\lambda_k\left( -I + \left(I + \frac{16}{\sigma^2}\Sigmabar_{\blambda,\sigma}^{\half} \Sigma_k \Sigmabar_{\blambda,\sigma}^{\half}\right)^{\negthickspace\negthinspace\half} \right). 
		\end{equation*}
		Furthermore, their Sinkhorn barycenter restricted to the manifold of Gaussians is also $\hat\sfb_{\blambda, p}^{\sfS,\sigma}(\rho_1,\ldots,\rho_K)=\euN(\mubar_{\blambda}, \tSigma_{\blambda,\sigma})\in\bbS_+^m$, where $\tSigma_{\blambda,\sigma}$ is 
		the unique solution of the following equation 
		\begin{equation*}
			\tSigma_{\blambda,\sigma} = \varphi_\sigma(\tSigma_{\blambda,\sigma}) \sumK\lambda_k\left[\Sigma_k^{\half}\left(I + \left(I + \frac{16}{\gamma^2}\Sigma_k^{\half}\tSigma_{\blambda,\sigma}\Sigma_k^{\half} \right)^{\negthickspace\negthinspace\half}\right)^{\negthickspace\negthinspace-1}\Sigma_k^{\half}\right]\varphi_\sigma(\tSigma_{\blambda,\sigma}), 
		\end{equation*}
		where $\varphi_\sigma(M) \coloneqq \left(I + \left(I + 16M^2/\sigma^2\right)^{\negthickspace\half}\right)^{\negthickspace\negthinspace\half}$. 
		Furthermore, if $\tSigma_{\blambda,\sigma}\in\bbS_{++}^m$, then $\tSigma_{\blambda,\sigma}$ is the unique solution of the equation 
		\begin{equation*}
			\tSigma_{\blambda,\sigma} = \frac{\sigma}{4}\left( -I + \left[\sumK\lambda_k \left( I + \frac{16}{\sigma^2}\tSigma_{\blambda,\sigma}^{\half} \Sigma_k \tSigma_{\blambda,\sigma}^{\half}\right)^{\negthickspace\negthinspace\half} \right]^{\negthinspace 2}\right)^{\negthickspace\negthickspace\half} . 
		\end{equation*}		
	\end{proposition}

	\subsection{An Example for \Cref{sec:learning_multiple_sources}}
	We now illustrate the formulation of stochastic barycentric optimization (SBO) through an example with Gaussian distribution which admits closed form for their $2$-Wasserstein barycenters. 
	Suppose that the source distributions are nondegenerate Gaussians, i.e., $\Prob_k = \euN(\mu_k, \Sigma_k)$, where $\mu_k\in\RR^m$ and $\Sigma_k\in\bbS_{+}^m$ for $k\in\setK$. Given the assumption that the source distributions are Gaussian, the empirical distributions are instead taken as $\Probhat_k = \euN(\muhat_k, \Sigmahat_k)$ for $k\in\setK$, where $\muhat_k = \frac{1}{n_k}\sum_{i=1}^{n_k} z_{k,i}$ and $\Sigmahat_k = \frac{1}{n_k}\sum_{i=1}^{n_k}(z_{k,i} - \muhat_k)(z_{k,i} - \muhat_k)^\top$ are their respective empirical means and empirical covariance matrices (which we assume to be full-rank). 
	
	Then, by \Cref{prop:location-scatter}, the $\blambda$-weighted $2$-Wasserstein barycenter of $\Probhat_1, \ldots, \Probhat_K$ is $\hat\sfb_{\blambda}(\Probhat_1, \ldots, \Probhat_K) = \euN(\mubar_{\blambda}, \Sigmabar_{\blambda})$, where 
	\[\mubar_{\blambda} = \sumK\lambda_k\mu_k \qquad \text{and} \qquad \Sigmabar_{\blambda} =  \argmin_{\Sigma\in\bbS_{++}^m} \ \sumK \lambda_k\sfB^2(\Sigma, \Sigmahat_k). \]
	
	Instead of directly minimizing the SBO objective $\hat{F}_{\blambda}^\sfb(x) \coloneqq \Ex_{\xi\sim\hat\sfb_{\blambda}(\Probhat_1, \ldots, \Probhat_K)}[\ell(x, \xi)]$, we first draw $N$ independent samples $\{\hat{\xi}_i\}_{i=1}^{N}$ from $\euN(\mubar_{\blambda}, \Sigmabar_{\blambda})$ and then minimize the ERM objective
	\[\hat{F}_{\mathsf{ERM}}^\sfb(x) \coloneqq\frac{1}{N}\sum_{i=1}^N \ell(x, \hat{\xi}_i). \]

	\subsection{Additional Results for \Cref{sec:Wass_bary_DRO}}
	\label{subsec:add_results_Wass_bary_ambiguity}
	Following the development of the Gelbrich ambiguity set in \Cref{sec:Gelbrich}, we define similar notions in the case of Wasserstein barycentric ambiguity set. 
	
	Let $\mu_1,\ldots,\mu_K\in\RR^m$ and $\Sigma_1,\ldots,\Sigma_K\in\bbS_{++}^m$. 
	The \emph{mean-covariance barycentric ambiguity set} is defined as 
	\[\teuU_{\varepsilon}\left( (\mu_k, \Sigma_k)_{k\in\setK}; \blambda\right)  \coloneqq 
	\left\{(\mu,\Sigma)\in\RR^m\times\bbS_{+}^m : \sumK\lambda_k\sfG^2((\mu,\Sigma), (\mu_k,\Sigma_k)) \le \varepsilon^2\right\}. \]
	Then, the \emph{Gelbrich barycentric ambiguity set} is defined as 
	\[\tilde{\euG}_\varepsilon\left( (\mu_k, \Sigma_k)_{k\in\setK}; \blambda\right)  \coloneqq \left\{\QQ\in\scrP_2(\Xi) : \left( \Ex_{\QQ}[\xi], \Cov_{\QQ}(\xi) \right) \in \euU_\varepsilon\left( (\mu_k, \Sigma_k)_{k\in\setK}\right)  \right\}. \]
	
	Now let $\Prob, \QQ_1,\ldots,\QQ_K\in\scrP_2(\Xi)$ with means $\mu,\mu_1,\ldots,\mu_K\in\RR^m$ and covariance matrices $\Sigma,\Sigma_1,\ldots,\Sigma_K\in\bbS_{++}^m$ respectively. Then, by the Gelbrich bound, we have 
	\[\sumK\lambda_k \sfW_2^2(\Prob, \QQ_k) \ge \sumK\lambda_k \sfG^2((\mu,\Sigma), (\mu_k,\Sigma_k)). \]
	The above inequality becomes an equality if $\Prob, \QQ_1,\ldots,\QQ_K$ belong to the same location-scatter family. 	
	Furthermore, this inequality implies 
	\[\teuW_{\varepsilon,2}(\QQ_1, \ldots,\QQ_K;\blambda) \subseteq \tilde{\euG}_\varepsilon\left( (\mu_k, \Sigma_k)_{k\in\setK}; \blambda\right). \]

	\subsection{Additional Results for \Cref{sec:Gelbrich}}
	\label{subsec:add_Gelbrich}
	The Gelbrich bound implies the Gelbrich ambiguity set is an outer approximation of the $p$-Wasserstein ambiguity set with $p\ge2$, i.e., $\euW_{\varepsilon, p}(\Probhat) \subseteq \euG_\varepsilon(\muhat, \Sigmahat)$ for every $p\ge2$. 
	This result immediately leads to an upper bound on the (optimal) worst-case risk by the (optimal) \emph{Gelbrich risk}  \citep[i.e., risk under the Gelbrich ambiguity set;][Corollary 1]{kuhn2019wasserstein,nguyen2021mean}. 
	\begin{corollary}[Worst-case risk bounds]\label{coro:risk_bounds}
		If the nominal distribution $\Probhat\in\scrP_2(\Xi)$ has mean $\muhat\in\RR^m$ and covariance matrix $\Sigmahat\in\bbS_{+}^m$, then, for every $p\ge2$, we have 
		\[
		(\forall\ell\in\euL)\quad \euR_{\euW_{\varepsilon,p}(\Probhat)}(\ell) \le \euR_{\euG_\varepsilon(\muhat, \Sigmahat)}(\ell) \qquad \text{and} \qquad \euR_{\euW_{\varepsilon,p}(\Probhat)}(\euL) \le \euR_{\euG_\varepsilon(\muhat, \Sigmahat)}(\euL). 
		\]
	\end{corollary}

	We now state the following proposition from \citet[Remark 9.5]{peyre2019computational}, which is proved in \citet[Theorem 6.1]{agueh2011barycenters} and previously known from \citet{knott1994generalization,ruschendorf2002n}. 
	
	\begin{proposition}\label{prop:location-scatter}
		Let $\Prob_0\in\scrP_2^{\ac}(\RR^m)$ and $\QQ_1, \ldots, \QQ_K\in\euF(\Prob_0)$ with means $\mu_1, \ldots,\mu_K\in\RR^m$ and covariance matrices $\Sigma_1, \ldots,\Sigma_K\in\bbS_{++}^m$ respectively. Then, for $\blambda\in\triangle^{K}$, the (unique) $\blambda$-weighted $2$-Wasserstein barycenter of $\QQ_1, \ldots, \QQ_K$ is $\bar{\QQ}_{\blambda,2}\in\euF(\Prob_0)$ with mean $\mubar_{\blambda}\in\RR^m$ and covariance matrix $\Sigmabar_{\blambda}\in\bbS_{++}^m$ given by, for each $k\in\setK$, 
		\begin{equation}\label{eqn:mu_Sigma}
			\mubar_{\blambda} = \sumK \lambda_k\mu_k \qquad \text{and} \qquad \Sigmabar_{\blambda} = \argmin_{\Sigma\in\bbS_{++}^m} \ \sumK \lambda_k\sfB^2(\Sigma, \Sigma_k), 
		\end{equation}
		where $\sfB(\cdot, \cdot)$ is the Bures--Wasserstein distance \eqref{eqn:Bures}. The covariance matrix $\Sigmabar_{\blambda}$ can be obtained by finding the unique positive definite fixed point of the equation 
		\begin{equation}\label{eqn:fixed_point}
			\Sigma = \sumK \lambda_k \left(\Sigma^{\sfrac12} \Sigma_k \Sigma^{\sfrac12}\right)^{\negthickspace\sfrac12}. 
		\end{equation}
	\end{proposition}

	\paragraph{Bures--Wasserstein Barycenter.} 	
	Another important example is the Bures--Wasserstein barycenter of $\Sigma_1,\ldots,\Sigma_K\in\bbS_{++}^m$, defined via the Bures--Wasserstein distance, which coincides with the Wasserstein barycenter of $\QQ_1, \ldots, \QQ_K$ in the same zero-mean location-scatter family with covariances $\Sigma_1,\ldots,\Sigma_K$ respectively, i.e., $\Sigmabar_{\blambda}$ in \eqref{eqn:mu_Sigma}. The issues of approximation, statistical inference and computational algorithms of the Bures--Wasserstein barycenter can be found in \citet{kroshnin2021multiplier,kroshnin2021statistical,chewi2020gradient}.

	\subsection{An Example of Tractability Results}	
	We now derive a tractability result from \Cref{prop:location-scatter,prop:location-scatter_ambiguity} for nominal distributions of location-scatter families and a quadratic loss function, similar to the one for WDRO \citep[Theorem 16]{kuhn2019wasserstein}. Assume that $\Xi=\RR^m$ and consider the quadratic loss function $\ell(\xi) = \dotp{\xi}{Q\xi} + 2\dotp{q}{\xi}$ with $Q\in\bbS^m$ and $q\in\RR^m$. If $\muhat_k\in\RR^m$ and $\Sigmahat_k\in\bbS_+^m$ for each $k\in\setK$, then the Gelbrich risk is equal to the optimal values of the following tractable SDP: 
	\begin{equation*}
		\euR_{\euG_\varepsilon(\mubar_{\blambda}, \Sigmabar_{\blambda})}(\ell) 
		= \inf \;\left\{ \alpha\left(\varepsilon^2 - \euclidnorm{\mubar_{\blambda}}^2 -\tr(\Sigmabar_{\blambda}) \right) + x + \tr(X) \right\},
	\end{equation*}
	subject to $\alpha\ge0$, $x\ge0$, $X\in\bbS_{+}^m$, 
	\[\begin{pmatrix}
		\alpha I - Q & q + \alpha\mubar_{\blambda} \\
		q^\top + \alpha\mubar_{\blambda}^\top & x
	\end{pmatrix} \succeq0, \quad
	\begin{pmatrix}
		\alpha I - Q & \alpha\Sigmabar_{\blambda}^{\sfrac12} \\
		\alpha\Sigmabar_{\blambda}^{\sfrac12} & X
	\end{pmatrix} \succeq 0. 
	\]
	If for each $k\in\setK$, $\Probhat_k\in\euF(\Prob_0)$ for some $\Prob_0\in\scrP_2^{\ac}(\RR^m)$ with mean $\muhat_k\in\RR^m$ and covariance matrix $\Sigmahat_k\in\bbS_{++}^m$, and $p=2$, then the optimal value of the above SDP, the worst-case risk \eqref{eqn:worst-case_risk}, and the Gelbrich risk all coincide. 
	
	To verify the above claim, we assume that $\QQ^\star\in\euF(\Prob_0)$ with mean $\muhat^\star$ and covariance $\Sigmahat^\star$ is the extremal distribution, i.e., 
	\[\QQ^\star = \argmax_{\QQ\in\euG_\varepsilon(\mubar_{\blambda}, \Sigmabar_{\blambda})}\euR_{\QQ}(\ell) \quad\text{and} \quad  \euR_{\QQ^\star}(\ell) = \euR_{\euG_\varepsilon(\mubar_{\blambda}, \Sigmabar_{\blambda})}(\ell). \]
	Note that $\QQ^\star\in\bar{\euW}_{\varepsilon,p}(\Probhat_1, \ldots, \Probhat_K; \blambda)$. Then the worst-case risk satisfies
	\[\euR_{\QQ^\star}(\ell) \le \euR_{\bar{\euW}_{\varepsilon,p}(\Probhat_1, \ldots, \Probhat_K; \blambda)}(\ell) 
	\le \euR_{\euG_\varepsilon(\mubar_{\blambda}, \Sigmabar_{\blambda})}(\ell), \]
	which follows from the Gelbrich bound and \Cref{thm:bary_risk_bounds}. Details of how to solve this SDP can be found in \citet{kuhn2019wasserstein}.

	\section{Proofs}
	\label{sec:proofs}
	
	\subsection{Preparatory Lemmas}
	We first state some useful preparatory lemmas.
	
	\begin{lemma}
		\label{lem:Wasserstein_ineq}
		For $1\le p < q <\infty$, we have 
		\[\sfW_p(\rho, \nu) \le \sfW_q(\rho, \nu) \]
		for any $\rho,\nu\in\scrP_q(\Omega)$ with $\Omega\subseteq\RR^m$. 
	\end{lemma}
	
	\begin{proof}[Proof of \Cref{lem:Wasserstein_ineq}]
		Since $\norm{x-y}^p$ is convex in $\norm{x-y}$ for all $p\in[1,+\infty)$, by Jensen's inequality, for $p\le q$ we have 
		\[\left(\int_{\Omega\times\Omega} \norm{x-y}^p \,\diff\pi(x,y)\right)^{\negthickspace\sfrac1p} \le \left(\int_{\Omega\times\Omega} \norm{x-y}^q \,\diff\pi(x,y)\right)^{\negthickspace\sfrac1q}, \]
		which implies that $ \sfW_p(\rho, \nu) \le \sfW_q(\rho, \nu) $. 
	\end{proof}
	
	\begin{lemma}
		\label{lem:bound}
		For $a,b\in\Rp$ and $p\in[1,+\infty)$, we have $(a+b)^p \le 2^{p-1}(a^p + b^p)$. 
	\end{lemma}
	
	\begin{proof}[Proof of \Cref{lem:bound}]
		by the convexity of $\Rp \ni s\mapsto s^p$, we get 
		\[\left(\frac{a+b}{2} \right)^p \le \frac{a^p+b^p}{2}, \]
		which is equivalent to $(a+b)^p \le 2^{p-1}(a^p + b^p)$. 
	\end{proof}

	\subsection{Proof of \Cref{thm:bounds}}
	\begin{proof}[Proof of \Cref{thm:bounds}]
		Assume that $\QQ_1,\ldots,\QQ_K\in\scrP_p(\Xi)$ have a $\blambda$-weighted $p$-Wasserstein barycenter $\bar{\QQ}_{\blambda,
			p}$, where $\blambda\in\triangle^K$. 
		Using the argument of \citet[Remark 3.1.2]{figalli2021invitation}, by the triangle inequality and \Cref{lem:bound}, for any $\Prob\in\scrP_p(\Xi)$ and $k\in\setK$, we have 
		\[\sfW_p^p(\Prob, \bar{\QQ}_{\blambda,p}) \le \left(\sfW_p(\Prob, \QQ_k) + \sfW_p(\bar{\QQ}_{\blambda,
			p}, \QQ_k)\right)^p \le 2^{p-1}\left(\sfW_p^p(\Prob, \QQ_k) + \sfW_p^p(\bar{\QQ}_{\blambda,
			p}, \QQ_k)\right). \]
		This implies that 
		\begin{align*}
			(\forall \Prob\in\scrP_p(\Xi))\quad \sfW_p^p(\Prob, \bar{\QQ}_{\blambda,p}) &= \sumK\lambda_k\sfW_p^p(\Prob, \bar{\QQ}_{\blambda,p}) \\
			&\le 2^{p-1}\sumK\lambda_k\sfW_p^p(\Prob, \QQ_k) + 2^{p-1}\sumK\lambda_k\sfW_p^p(\bar{\QQ}_{\blambda,
				p}, \QQ_k) \\
			&\le 2^{p-1}\sumK\lambda_k\sfW_p^p(\Prob, \QQ_k)  + 2^{p-1}\sumK\lambda_k\sfW_p^p(\Prob, \QQ_k) \\
			&=2^{p}\sumK\lambda_k\sfW_p^p(\Prob, \QQ_k), 
		\end{align*}
		where the second inequality follows from the definition of the $\blambda$-weighted $p$-Wasserstein barycenter. 
		
		Consequently, for any $\varepsilon\ge0$, 
		\[\sumK\lambda_k\sfW_p^p(\Prob, \QQ_k) \le \varepsilon \implies \sfW_p^p(\Prob, \bar{\QQ}_{\blambda,p}) \le 2^p\cdot\varepsilon, \]
		which implies the desired result. 
	\end{proof}

	\subsection{Proof of \Cref{thm:conc_1}}
	\begin{proof}[Proof of \Cref{thm:conc_1}]
		First note that the $2$-Wasserstein space $\scrW_2(\RR^m) = (\scrP_2(\RR^m), \sfW_2)$ is positively curved in the sense of Alexandrov \citep[\S7.3]{ambrosio2005gradient}. Then, by \citet[Theorem 12]{le2019fast}, we have, for any $\beta\in(0,1)$, 
		\[\Prob^{n}\left\{\sfW_2^2\left(\hat{\sfb}_{\One/K,2}(\Probhat^n), \hat{\sfb}_K^\star\right) \le \frac{c_1}{n}\log\left(\frac2\beta\right)\right\} \ge 1-\beta-\e^{-c_2n}, \]
		where $\hat{\sfb}_{\One/K,2}(\Probhat^n) = \sfb_2(\hat{\rho}_{n,K})$, and $c_1, c_2\in\Rpp$ are independent of $n$. By the definition of $\bar{\euW}_{\varepsilon,2}(\Probhat^n; \One/K)$, we also have 		
		\[\Prob^{n} \left\{\hat{\sfb}_K^\star \in \bar{\euW}_{\varepsilon,2}(\Probhat^n; \One/K) \right\} = \Prob^{n}\left\{\sfW_2\left(\hat{\sfb}_{\One/K,2}(\Probhat^n), \hat{\sfb}_K^\star\right) \le \varepsilon\right\}, \]
		hence the desired result. 	
	\end{proof}
	
	\begin{remark}
		Note that \citet{adve2020nonexpansiveness} provides a heuristic argument for the non-negative curvature of $\scrW_p(\RR^m)$ when $p\in(1, p(m))$, where $p(m) > 1$ is close to $1$ (e.g., $p(m) = 1+1/\scrO(m^2 \log m)$). Consequently, \Cref{thm:conc_1} should also hold for $p\in(1, p(m))$. For $p\notin(1,p(m))\cup\{2\}$, it remains unclear whether $\scrW_p(\RR^m)$ is positively curved even though it is expected. In this case, the general concentration inequality in \citet[Theorem 12]{le2019fast} involving various abstract constants cannot be easily simplified. 
	\end{remark}

	\subsection{Proof of \Cref{thm:finite_1}}
	\begin{proof}[Proof of \Cref{thm:finite_1}]
		By \Cref{thm:conc_1}, the inequality $\Prob^{n} \left\{\hat{\sfb}_K^\star \in \bar{\euW}_{\varepsilon,2}(\Probhat^n; \One/K) \right\} \ge 1-\beta-\e^{-c_2n}$ immediately implies that 
		\[\euR_{\hat{\sfb}_K^\star}(\ell)\le \sup_{\Prob\in\bar{\euW}_{\varepsilon,2}(\Probhat^n; \One/K)} \euR_\Prob(\ell) \eqqcolon \euR_{\bar{\euW}_{\varepsilon_n,2}(\Probhat^n; \One/K)}(\ell)\]
		with probability at least $1-\beta-\e^{-c_2n}$, where $c_2>0$ is independent of $n$. 		
	\end{proof}

	\subsection{Proof of \Cref{thm:consistency_1}}
	To prove \Cref{thm:consistency_1}, we need the following lemma which is a slight modification of \citet[Lemma 3.7]{esfahani2018data}. 
	\begin{lemma}
		\label{lem:conv_dist}
		Let $K\in\NN^*$ be finite and fixed. 
		Suppose that $\Prob\in\scrW_2(\scrP_2(\Xi))$ is sub-Gaussian, $\beta_n\in(0,1)$ and $\varepsilon_n = \varepsilon_n(\beta_n)\in\Rpp$ for $n\in\NN^*$, satisfies $\sum_{n=1}^\infty \beta_n < \infty$ and $\lim_{n\to\infty}\varepsilon_n(\beta_n) = 0$, then any sequence $\hat{\QQ}_n \in \bar{\euW}_{\varepsilon_n(\beta_n),2}(\Probhat^n; \One/K)$, $n\in\NN^*$, where $\hat{\QQ}_n$ may depend on the observations, converges in the $\sfW_2$ distance to $\hat{\sfb}_K^\star$ almost surely with respect to $\Prob^\infty$, i.e,
		\[\Prob^\infty\left\{ \lim\limits_{n\to\infty} \sfW_2\left( \hat{\sfb}_K^\star, \hat{\QQ}_n \right)\right\} = 1.\]
	\end{lemma}
	
	\begin{proof}[Proof of \Cref{lem:conv_dist}]
		For any $\hat{\QQ}_n \in \bar{\euW}_{\varepsilon_n(\beta_n),2}(\Probhat^n; \One/K)$, the triangle inequality implies
		\[\sfW_2\left( \hat{\sfb}_K^\star, \hat{\QQ}_n \right) \le \sfW_2\left( \hat{\sfb}_K^\star, \hat{\sfb}_{\One/K,2}(\Probhat^n) \right) + \sfW_2\left(\hat{\sfb}_{\One/K,2}(\Probhat^n), \hat{\QQ}_n \right) \le \sfW_2\left( \hat{\sfb}_K^\star, \hat{\sfb}_{\One/K,2}(\Probhat^n) \right) + \varepsilon_n(\beta_n). \]
		In addition, by \Cref{thm:finite_1}, we have $\Prob^{n} \left\{\sfW_2\left( \hat{\sfb}_K^\star, \hat{\sfb}_{\One/K,2}(\Probhat^n) \right) \le \varepsilon_n(\beta_n)\right\} \ge 1-\beta_n-\e^{-c_2n}$, which implies 
		\[\Prob^n\left\{ \sfW_2\left( \hat{\sfb}_K^\star, \hat{\QQ}_n \right) \le 2 \varepsilon_n(\beta_n)\right\} \ge 1-\beta_n-\e^{-c_2n}. \]
		Since $\sum_{n=1}^\infty \beta_n < \infty$ and $\sum_{n=1}^\infty \e^{-c_2n} < \infty$ as $c_2>0$, invoking the (second) Borel--Cantelli lemma \citep[Theorem 4.2.4]{chung2001course} yields
		\[\Prob^\infty\left\{\sfW_2\left( \hat{\sfb}_K^\star, \hat{\QQ}_n \right) \le \varepsilon_n(\beta_n) \;\text{for sufficiently large $n$}\right\} = 1.\] 
		Since $\varepsilon_n(\beta) \to 0$ as $n\to\infty$, we conclude that $\lim_{n\to\infty} \sfW_2\left( \hat{\sfb}_K^\star, \hat{\QQ}_n \right) $ $\Prob^\infty$-almost surely. 		
	\end{proof}

	Now we are ready to prove \Cref{thm:consistency_1}. 
	\begin{proof}[Proof of \Cref{thm:consistency_1}]
		Let $K\in\NN^*$ be finite and fixed. 
		Let us define 
		\[\ell_{n,K}^\star \coloneqq \argmin_{\ell\in\euL} \euR_{\bar{\euW}_{\varepsilon,2}(\Probhat^n; \One/K)}\quad\text{and} \quad \euR_{\hat{\sfb}_K^\star}(\euL) \coloneqq \inf_{\ell\in\euL} \euR_{\hat{\sfb}_K^\star}(\ell). \]
		Since $\ell_{n,K}^\star\in\euL$, we have $ \euR_{\hat{\sfb}_K^\star}(\ell_{n,K}^\star) \le \euR_{\hat{\sfb}_K^\star}(\euL) $. 
		\Cref{thm:conc_1} implies 
		\begin{align*}
			\Prob^n\left\{ \euR_{\hat{\sfb}_K^\star}(\ell_{n,K}^\star) \le \euR_{\hat{\sfb}_K^\star}(\euL) \le \euR_{\bar{\euW}_{\varepsilon_n(\beta_n),2}(\Probhat^n; \One/K)}(\ell_{n,K}^\star) \right\}  &\ge \Prob^{n} \left\{\hat{\sfb}_K^\star \in \bar{\euW}_{\varepsilon_n(\beta_n),2}(\Probhat^n; \One/K) \right\} \\
			&\ge 1-\beta-\e^{-c_2n}
		\end{align*}
		for all $n\in\NN^*$. Since $\sum_{n=1}^\infty\beta_n<\infty$ and $\sum_{n=1}^\infty \e^{-c_2n} < \infty$ as $c_2>0$, invoking the (second) Borel--Cantelli lemma again yields
		\[\Prob^\infty\left\{ \euR_{\hat{\sfb}_K^\star}(\ell_{n,K}^\star) \le \euR_{\hat{\sfb}_K^\star}(\euL) \le \euR_{\bar{\euW}_{\varepsilon_n(\beta_n),2}(\Probhat^n; \One/K)}(\ell_{n,K}^\star) \;\text{for sufficiently large $n$}\right\} = 1.\] 
		Thus, it remains to show that $\limsup_{n\to\infty} \euR_{\bar{\euW}_{\varepsilon_n(\beta_n),2}(\Probhat^n; \One/K)}(\ell_{n,K}^\star) \le \euR_{\hat{\sfb}_K^\star}(\ell_{n,K}^\star)$ with probability one. This part is more subtle and involves overwhelming technicalities---we refer to the proof of \citet[Theorem 3.6]{esfahani2018data} as our proof follows exactly the same steps. 

		The second part of \Cref{thm:consistency_1} follows exactly the same procedures as that of the first part. By considering that $n$ has already been taken to $\infty$, results similar to \Cref{thm:conc_1,thm:finite_1} also hold, by replacing $n$ with $K$, $\Probhat^n$ with $\Probhat^\star$, $\hat{\sfb}_K^\star$ with $\sfb^\star$, etc. 
	\end{proof}

	\subsection{Proof of \Cref{prop:location-scatter_ambiguity}}
	\begin{proof}[Proof of \Cref{prop:location-scatter_ambiguity}]
		Since $\QQ_1, \ldots, \QQ_K$ belong to the same location-scatter family $\euF(\Prob_0)$ with $\Prob_0\in\scrP_2^{\ac}(\RR^m)$, so is their $\blambda$-weighted $2$-Wasserstein barycenter $\bar{\QQ}_{\blambda,2}$, since location-scatter families are closed for barycenters \citep[Theorem 3.8]{alvarez2018wide}.
		
		Let us recall the definition of the $2$-Wasserstein ambiguity set 
		\[\bar{\euW}_{\varepsilon, 2}^{\mathsf{ls}}(\QQ_1, \ldots, \QQ_K; \blambda) 
		\coloneqq  \left\{\Prob\in \euF(\Prob_0) : \sfW_2(\Prob, \bar{\QQ}_{\blambda,2}) \le \varepsilon \right\}. \]
		Recall that the $\sfW_2$ distance between two distributions of the same location-scatter family is the Gelbrich distance, i.e., $\sfW_2(\Prob, \bar{\QQ}_{\blambda,2}) = \sfG((\mu, \Sigma), (\bar{\mu}_{\blambda}, \bar{\Sigma}_{\blambda}))$, where $\mu$ and $\Sigma$ are the mean and the covariance matrix of $\Prob\in\euF(\Prob_0)$ respectively. Therefore, we have 
		\begin{align*}
			\bar{\euW}_{\varepsilon, 2}^{\mathsf{ls}}(\QQ_1, \ldots, \QQ_K; \blambda) 
			&= \left\{\Prob\in \euF(\Prob_0) : \sfG((\mu, \Sigma), (\bar{\mu}_{\blambda}, \bar{\Sigma}_{\blambda})) \le \varepsilon \right\} \\
			&= \left\{\Prob\in \scrP_2(\Xi) : \sfG((\mu, \Sigma), (\bar{\mu}_{\blambda}, \bar{\Sigma}_{\blambda})) \le \varepsilon \right\} \cap \euF(\Prob_0) \\
			&= \euG_{\varepsilon}(\bar{\mu}_{\blambda}, \bar{\Sigma}_{\blambda}) \cap \euF(\Prob_0).
		\end{align*}		
	\end{proof}

	\subsection{Proof of \Cref{thm:bary_risk_bounds}}
	\begin{proof}[Proof of \Cref{thm:bary_risk_bounds}]
		Let $\Prob\in\scrP_2(\Xi)$ with mean $\mu\in\RR^m$ and covariance matrix $\Sigma\in\bbS_{+}^m$, and let $\Probhat_1,\ldots,\Probhat_K\in\scrP_2^{\ac}(\Xi)$ with means $\mu_1,\ldots,\mu_K\in\RR^m$ and covariance matrices $\Sigma_1,\ldots,\Sigma_K\in\bbS_{++}^m$ respectively. 
		By the Gelbrich bound, we have 
		\[\sfW_2(\Prob, \hat{\sfb}_{\blambda,2}(\Probhat_1,\ldots,\Probhat_K)) \ge \sfG((\mu, \Sigma), (\bar{\mu}_{\blambda}, \bar{\Sigma}_{\blambda})), \]
		where $\bar{\mu}_{\blambda}$ and $\bar{\Sigma}_{\blambda}$ are the mean and the covariance matrix of $\hat{\sfb}_{\blambda,2}(\Probhat_1,\ldots,\Probhat_K)$. 
		The two worst-case risk bounds immediately follow from this inequality due to the definitions of the worst-case risk and the optimal worst-case risk. 		
	\end{proof}

	\subsection{Proof of \Cref{coro:risk_bounds}}	
	\begin{proof}[Proof of \Cref{coro:risk_bounds}]
		The Gelbrich bound and \Cref{lem:Wasserstein_ineq} together imply that, for any $\Prob,\Probhat\in\scrP_2(\Xi)$ with means $\mu,\muhat\in\RR^m$ and covariance matrices $\Sigma, \Sigmahat\in\bbS_{++}^m$, we have
		\[\sfG((\mu,\Sigma), (\muhat,\Sigmahat)) = \sfW_2(\rho, \nu) \le \sfW_p(\rho, \nu). \]
		Consequently, we have the inclusion property 
		\[\euW_{\varepsilon, p}(\Probhat) \subseteq \euG_\varepsilon(\muhat, \Sigmahat)\]
		for every $p\ge2$. 		
		Thus, according to the definitions of the worst-case risk \eqref{eqn:worst-case_risk} and the optimal worst-case risk \eqref{eqn:dro}, the desired results follow. 		
	\end{proof}

	\subsection{Proof of \Cref{prop:location-scatter}}
	\begin{proof}[Proof of \Cref{prop:location-scatter}]
		See e.g., \citet[Theorem 6.1]{agueh2011barycenters} and \citet{knott1994generalization,ruschendorf2002n}. 
	\end{proof}

	\section{Experimental Details} 
	\label{sec:add_DRMLE}
	In this section, we give further details about \Cref{sec:DRMLE}. 
	
	\subsection{The Wasserstein Barycentric Shrinkage Estimator}
	The following theorem indicates the way to compute the solution $X^\star$ of the DRMLE problem \eqref{eqn:DRMLE}. 
	\begin{theorem}\label{thm:DRMLE}
		Assume that $\varepsilon>0$ and $\Sigmahat_k\in\bbS_{+}^m$ for each $k\in\setK$ with all least one of the $\Sigmahat_k$'s in $\bbS_{++}^m$, and that the $\blambda$-weighted Bures--Wasserstein barycenter $\bar{\Sigma}_{\blambda}$ of $\Sigmahat_1, \ldots, \Sigmahat_K$ admits the spectral decomposition $\bar{\Sigma}_{\blambda} = \sum_{j=1}^m \zeta_j v_j v_j^\top$ with eigenvalues $\zeta_j \in\Rp$ and corresponding orthonormal eigenvectors $v_j\in\RR^m$, $j\in\set{m}$. Then the unique minimizer of the DRMLE problem \eqref{eqn:DRMLE} is given by $X^\star = \sum_{j=1}^m x_j^\star v_j v_j^\top$, where, for each $j\in\set{m}$, 
		\begin{equation*}
			x_j^\star = \chi^\star\left[1-\frac12\left(\sqrt{\zeta_j^2(\chi^\star)^2 + 4\zeta_j^2\chi^\star} - \zeta_j\chi^\star\right) \right],
		\end{equation*}
		and $\chi^\star>0$ is the unique positive solution of the equation 
		\begin{equation}\label{eqn:chi}
			\left(\varepsilon^2 - \frac12\sum_{j=1}^m \zeta_j\right)\chi - m +\frac12\sum_{j=1}^m \sqrt{\zeta_j^2\chi^2 + 4 \zeta_j\chi} = 0. 
		\end{equation}
	\end{theorem}
	
	\begin{proof}[Proof of \Cref{thm:DRMLE}]
		According to the formulation of the DRMLE problem \eqref{eqn:DRMLE}. The above theorem is simply \citet[Theorem 3.1]{nguyen2022distributionally} with $\Sigmahat$ replaced by $\bar{\Sigma}_{\blambda}$. 		
	\end{proof}
	
	Recall that $\Sigmabar_{\blambda}$ can be obtained by finding the unique positive definite fixed point of the equation \eqref{eqn:fixed_point}. Thus, to approximate the barycenter $\bar{\Sigma}_{\blambda}$, one iterates 
	\[(\forall t\in\NN^*)\quad S_{t+1} = S_t^{-\sfrac12}\left(\sumK\lambda_k (S_t^{\sfrac12}\Sigma_k S_t^{\sfrac12})^{\sfrac12}\right)^{\negthickspace\negthinspace 2}S_t^{-\sfrac12}, \]
	which gives $\lim_{t\to\infty}S_t = \bar{\Sigma}_{\blambda}$. 
	Details of this iterative scheme can be found in \citet{alvarez2016fixed}.

	\subsection{The Averaged Linear Shrinkage Estimator}
	A na\"{i}ve estimator is constructed by simply replacing the sample covariance matrix in the widely-used linear shrinkage estimator by the $\blambda$-weighted average of the sample covariance matrices $\Sigmahat_1, \ldots, \Sigmahat_K$, defined as follows. 	
	\begin{definition}[Averaged linear shrinkage estimator]
		Given $K$ empirical covariance matrices $\Sigmahat_1,\ldots,\Sigmahat_K\in\bbS_{++}^m$, the $\blambda$-weighted linear shrinkage estimator for the precision matrix $X\in\bbS_{++}^m$ is defined by 		
		\[X^\star = \left[(1-\alpha)\sumK\lambda_k\Sigmahat_k + \alpha\sumK\lambda_k\Diag(\Sigmahat_k)\right]^{-1}, \]
		where $\alpha\in[0,1]$, $\blambda\in\triangle^K$ and $\Diag(A)$ is the diagonal matrix with the same diagonal of the square matrix $A\in\RR^{m\times m}$. 		
	\end{definition}
	Let us recall that the $\blambda$-weighted average of the sample covariance matrices $\Sigmahat_1, \ldots, \Sigmahat_K$ is indeed the Frobenius barycenter of $\Sigmahat_1, \ldots, \Sigmahat_K$: 	
	\[\sumK\lambda_k\Sigmahat_k = \argmin_{\Sigma\in\bbS_{++}^m} \ \sumK\lambda_k\fronorm{\Sigma - \Sigmahat_k}^2. \]

	\subsection{The Averaged $L_1$-Regularized Maximum Likelihood Estimator}
	An averaged $L_1$-regularized maximum likelihood estimator of the precision matrix can be obtained from 
	simply minimizing a $\blambda$-weighted loss function of the original $L_1$-regularized maximum likelihood estimation problem.

	\begin{definition}[Averaged $L_1$-regularized maximum likelihood estimator]
		Let us recall that the original $L_1$-regularized maximum likelihood estimation problem takes the following objective: 
		\[\minimize_{X\in\bbS_{++}^m} \ g(X,\Sigmahat) \coloneqq -\log\det X + \dotpF{\Sigmahat}{X} + \tau\onenorm{X}, \]
		where $\tau\ge0$. Then, given $K$ empirical covariance matrices $\Sigmahat_1,\ldots,\Sigmahat_K\in\bbS_{++}^m$, the $\blambda$-weighted $L_1$-regularized maximum likelihood estimator for the precision matrix $X\in\bbS_{++}^m$ is defined by 
		\[X^\star = \argmin_{X\in\bbS_{++}^m} \  \sumK\lambda_k g(X, \Sigmahat_k) = -\log\det X + \sumK\lambda_k\dotpF{\Sigmahat_k}{X} + \tau\onenorm{X} = g\left( X, \sumK\lambda_k\Sigmahat_k\right) , \]
		where $\tau\ge0$ and $\blambda\in\triangle^K$. 		
	\end{definition}

	\subsection{The Sinkhorn Barycentric Shrinkage Estimator}
	\label{subsec:SBSE}
	The Sinkhorn Barycentric Shrinkage Estimator (SBSE) is simply an estimator similar to WBSE with the $2$-Wasserstein barycenter replaced by the Sinkhorn barycenter, which is used because of the non-existence of the $2$-Wasserstein barycenter under the high-dimensional setting and computational consideration used in simulations. Theoretical treatment of this estimator is left for future work. We do not use the entropic-$2$-Wasserstein barycenter in simulations since it does \emph{not} make much sense \citep[see][Remark 4]{minh2022entropic}. Also recall from \Cref{prop:ent_barycenter_Gaussian} for the Sinkhorn barycenter of Gaussians (restricted to the manifold of Gaussians).

	\subsection{Simulation Settings}	
	All experiments were run with a laptop with Intel Core i7-7700HQ CPU (2.80 GHz) and 32GB RAM, using Python 3.9 with libraries \texttt{numpy} \citep{numpy2020array}, \texttt{scipy} \citep{scipy2020} and \texttt{scikit-learn} \citep{scikit-learn2011}. 
	
	The equation \eqref{eqn:chi} is solved via the Netwon--Raphson method in \texttt{scipy}. 
	We give the choice of tuning parameters in \Cref{table:tuning}, which are obtained based on grid search. 	
	\begin{table}[h]
		\caption{Tuning parameters of LS, $L_1$ and WBSE.}
		\label{table:tuning}
		\begin{center}
			\begin{small}
				\begin{tabular}{ccccc}
					\toprule
					$n$ & $K$ & $\alpha$ & $\tau$ & $\varepsilon$ \\
					\midrule						  
					& 25 &  0.1  &  0.1  &  0.3 \\
					50 & 50 &  0.1  &  0.1  &  0.3 \\
					& 100 & 0.1   & 0.1   &  0.3 \\
					\midrule
					& 25 &   0.1 & 0.1  &  0.03 \\
					100 & 50 &  0.1  &  0.1  & 0.03  \\
					& 100 &  0.1  &  0.1  &  0.03 \\
					\midrule
					& 25 &  0.1  &  0.1  & 0.03  \\
					200 & 50 & 0.1   &  0.1  &  0.03 \\
					& 100 &  0.1  &  0.1  & 0.005  \\
					\bottomrule
				\end{tabular}
			\end{small}
		\end{center}
	\end{table}

	\begin{remark}
		While we treat the choice of $\varepsilon$ as a tuning parameter in simulations, in \citet{blanchet2019optimal}, the optimal distributional uncertainty size $\varepsilon = \varepsilon_n$ is studied as a function of the sample size $n$ for the Wasserstein Shrinkage Estimator \citep{nguyen2022distributionally}. \citet{blanchet2019optimal} prove that $\varepsilon_n$ should scale at rate $\varepsilon_n= \varepsilon^\star n^{-1}(1+o(1))=\scrO(n^{-1})$, which aligns with the empirical findings of \citet{nguyen2022distributionally}. 
		This is as opposed to the theoretical rate of $\scrO(n^{-\sfrac12})$. It is interesting to find the optimal scaling of $\varepsilon$ in our proposed Wasserstein Barycentric Shrinkage Estimator in terms of both $n$ and $K$, 
		and is left for future work. 
	\end{remark}

	\section{Simulations under High-Dimensional Setting}
	In this section, we study the performance of the proposed Sinkhorn barycentric shrinkage estimator (see \Cref{subsec:SBSE}) under the high-dimensional setting ($m> n$). Let us recall that under the high-dimensional setting, the sample covariance matrices $\Sigmahat_k$'s are all rank-deficient, so the empirical (unregularized) $2$-Wasserstein barycenter of $K$ Gaussians does not exist. We thus resort to the Sinkhorn barycenter with entropic regularization strength $\sigma>0$. In particular, we choose $m=20$, $n\in\{5,10,15\}$, $K\in\{25,50,100\}$ and $\sigma=0.1$. 
	The Stein losses of the estimators are given in \Cref{table:stein_losses_estimators}, averaged over $20$ independent trials. 
	\begin{table}[h]
		\caption{Stein losses of  LS, $L_1$ and SBSE.}
		\label{table:stein_losses_estimators}
		\begin{center}
			\begin{small}
				\begin{tabular}{ccccc}
					\toprule
					$n$ & $K$ & LS & $L_1$ & SBSE  \\
					\midrule						  
					& 25 &  7.61 $\pm$ 0.73  &  8.43 $\pm$ 0.79  & 2.72 $\pm$ 0.23  \\						 
					5 & 50 &  7.75 $\pm$ 0.59  & 8.67 $\pm$ 0.63  & 2.54 $\pm$ 0.21 \\				
					& 100 &  7.70 $\pm$ 0.47  & 8.67 $\pm$ 0.51 & 2.43 $\pm$ 0.15 \\				
					\midrule
					& 25 &  7.12 $\pm$ 0.69  &  7.99 $\pm$ 0.74  &  1.94 $\pm$ 0.29 \\				
					10 & 50 &  7.12 $\pm$ 0.44  &  8.04 $\pm$ 0.47  & 1.37 $\pm$ 0.17 \\								
					& 100 & 7.20 $\pm$ 0.35 & 8.15 $\pm$ 0.38 & 1.07 $\pm$ 0.10 \\				
					\midrule
					& 25 &  6.96 $\pm$ 0.55  &  7.84 $\pm$ 0.60  &  1.91 $\pm$ 0.27  \\				
					15 & 50	&  6.99 $\pm$ 0.40  &  7.90 $\pm$ 0.43  &  1.12 $\pm$ 0.16 \\				
					& 100 & 6.97 $\pm$ 0.27 & 7.90 $\pm$ 0.29 & 0.71 $\pm$ 0.10 \\				
					\bottomrule
				\end{tabular}
			\end{small}
		\end{center}
	\end{table}
	
	Similar to WBSE under the low-dimensional setting, we observe that SBSE also outperforms the other two estimators by a large margin. The performance of SBSE also improves as $n$ and $K$ increase. 
	
	We also give the choice of tuning parameters in \Cref{table:tuning_sinkhorn}, which are obtained based on grid search. 	
	\begin{table}[h]
		\caption{Tuning parameters of LS, $L_1$ and SBSE.}
		\label{table:tuning_sinkhorn}
		\begin{center}
			\begin{small}
				\begin{tabular}{ccccc}
					\toprule
					$n$ & $K$ & $\alpha$ & $\tau$ & $\varepsilon$ \\
					\midrule						  
					& 25 & 0.1 & 0.1 & 1 \\
					5 & 50 & 0.1 & 0.1 & 1 \\
					& 100  & 0.1 & 0.1 & 1 \\
					\midrule
					& 25  & 0.1 & 0.1 & 0.5 \\
					10 & 50 & 0.1 & 0.1 & 0.5 \\
					& 100 & 0.1 & 0.1 & 0.5 \\
					\midrule
					& 25 & 0.1  &  0.1  & 0.3 \\
					15 & 50 & 0.1  &  0.1  & 0.3 \\
					& 100 & 0.1  &  0.1  & 0.3 \\
					\bottomrule
				\end{tabular}
			\end{small}
		\end{center}
	\end{table}

	\paragraph{The Effect of Entropic Regularization Strength $\sigma$ in SBSE.}
	We also numerically study how different entropic regularization strengths in the Sinkhorn barycenter affect the performance of the proposed Sinkhorn barycentric shrinkage estimator. We choose $m=20$, $n=5$, $K=25$, $\alpha=0.1$, $\tau=0.1$, $\varepsilon=0.8$ and $\sigma\in\{0.01, 0.1, 1, 10, 100\}$. 
	
	\begin{table}[h]
		\caption{Stein losses of LS, $L_1$ and SBSE with different $\sigma$'s.}
		\label{table:stein_losses_gamma}
		\begin{center}
			\begin{small}
				\begin{tabular}{cccc}
					\toprule
					$\sigma$ & LS & $L_1$ & SBSE \\
					\midrule						  					 
					0.01 &  &    & 4.47 $\pm$ 0.35 \\	
					0.1 &  &    & 2.03 $\pm$ 0.22 \\				
					1 & 7.61 $\pm$ 0.73 &  8.43 $\pm$ 0.79  & 5.58 $\pm$ 0.36 \\				
					10 &  &    &  9.82 $\pm$ 0.71 \\			
					100 &  &    &  11.70 $\pm$ 0.88 \\			
					\bottomrule
				\end{tabular}
			\end{small}
		\end{center}
	\end{table}
	We observe that, given a fixed set of $(m,n,K,\varepsilon)$, smaller $\sigma$ (i.e., closer approximation to the unregularized $2$-Wasserstein barycenter) does not necessarily yield lower Stein loss. However, as the strength of entropic regularization grows, the performance of SBSE decays sharply, which could be even worse than the averaged linear shrinkage and averaged $L_1$-regularized maximum likelihood estimators.

\end{document}